\documentclass{article}
\usepackage{fullpage}

\usepackage{natbib}
\usepackage{xspace}
\usepackage{authblk}
\usepackage{times}
\usepackage{algorithm, algorithmic}
\usepackage{mkolar_definitions}
\usepackage{dsfont}
\usepackage{authblk}

\title{Efficient Active Learning Halfspaces with Tsybakov Noise:\\ A Non-convex Optimization Approach}

\DeclareMathOperator{\sign}{sign}

\DeclareMathOperator{\err}{err}
\DeclareMathOperator{\polylog}{polylog}
\DeclareMathOperator{\poly}{poly}

\DeclareMathOperator{\dif}{d}
\newcommand\normx[1]{\left\Vert#1\right\Vert}
\renewcommand{\ind}{\mathds{1}}
\newcommand{\opt}{\mathrm{opt}}

\newcommand{\Ber}{\mathrm{Bernoulli}}

\newcommand{\apsgd}{\ensuremath{\textsc{Active-PSGD}}\xspace}
\newcommand{\afo}{\ensuremath{\textsc{Active-FO}}\xspace}

\def\shownotes{1}  \ifnum\shownotes=1
\newcommand{\authnote}[2]{$\ll$\textsf{\footnotesize #1 notes: #2}$\gg$}
\else
 \newcommand{\authnote}[2]{}
\fi

\author[]{Yinan Li \thanks{Email: yinanli@arizona.edu}}
\author[]{Chicheng Zhang \thanks{Email: chichengz@cs.arizona.edu} }
\affil[]{University of Arizona}

\begin{document}

\maketitle

\begin{abstract}%
We study the problem of computationally and label efficient PAC active learning $d$-dimensional halfspaces with Tsybakov Noise~\citep{tsybakov2004optimal} under structured unlabeled data distributions. Inspired by~\cite{diakonikolas2020learning}, we prove that any approximate first-order stationary point of a smooth nonconvex loss function yields a halfspace with a low excess error guarantee. 
In light of the above structural result, we design a nonconvex optimization-based algorithm with a label complexity of $\tilde{O}(d (\frac{1}{\epsilon})^{\frac{8-6\alpha}{3\alpha-1}})$\footnote{In the main body of this work, we use $\tilde{O}(\cdot), \tilde{\Theta}(\cdot)$ to hide factors of the form $\polylog(d, \frac{1}{\epsilon}, \frac{1}{\delta})$}, under the assumption that the Tsybakov noise parameter $\alpha \in (\frac13, 1]$, which narrows down the gap between the label complexities of the previously known efficient passive or active  algorithms~\citep{diakonikolas2020polynomial,zhang2021improved} and the information-theoretic lower bound in this setting.
\end{abstract}

\section{Introduction}

Active learning~\citep{settles2009active} is a practical machine learning paradigm motivated by the expensiveness of label annotation costs and the wide availability of unlabeled data. 
Consider the binary classification setting, where given an instance space $\Xcal$ and a binary label space $\Ycal = \cbr{-1,+1}$ and a data distribution $D$ over $\Xcal \times \Ycal$, we 
would like to learn a classifier that accurately predicts the labels of examples drawn from $D$. 
As the performance measure of a classifier $h$, we define its error rate to be $\err(h) := \PP_{(x,y) \sim D}(h(x) \neq y)$.
Given access to unlabeled examples and the ability to interactively query a labeling oracle (oftentimes a human annotator), an active learning algorithm aims to output a model $\hat{h}$ from a hypothesis class $\Hcal$ that has a low error rate with a small number of label queries. It has been shown both theoretically~\citep{settles2009active,dasgupta2005coarse,hanneke2011rates,hanneke2014theory,balcan2007margin,balcan2013active,hanneke2015minimax,zhang2014beyond}  and empirically~\citep[e.g.][]{siddhant2018deep,dor2020active}
that, under many learning settings, by utilizing interaction, active learning algorithms can enjoy much better label efficiency compared with conventional supervised learning.

In this work, we study the problem of computationally and label efficient PAC active learning halfspaces~\citep{valiant1985learning} with noise under structured unlabeled data distributions, where the hypothesis class $\Hcal := \cbr{h_w(x) := \sign(\inner{w}{x}): w \in \RR^d}$ is the set of linear classifiers, and $D_X$, the marginal distribution of $D$ over $\Xcal$, satisfies certain structural assumptions~\citep{diakonikolas2020learning} (see Definition~\ref{def:well-behaved} in Section~\ref{sec:prelims}). 
The goal of the learner is to $(\epsilon, \delta)$-PAC learn $\Hcal$ and $D$, i.e.
to output a classifier $\hat{h}$ such that with probability at least $1-\delta$, its excess error, $\err(\hat{h}) - \min_{h' \in \Hcal} \err(h')$ is at most $\epsilon$; the total number of label queries the learners makes as a function of $\epsilon, \delta$ is referred to as its \emph{label complexity}. 

In this work, the specific label noise condition we are interested in is the Tsybakov noise condition (TNC)~\citep{mammen1999smooth,tsybakov2004optimal}, stated below:
\begin{definition}[Tsybakov noise condition]
Given $A > 0$ and $\alpha \in (0,1]$, a distribution $D$ over $\RR^d \times \cbr{-1,+1}$ is said to satisfy the {\em $(A, \alpha)$-Tsybakov noise condition} with respect to halfspace $w^\star \in \RR^d$, if for all $t \in [0,\frac12]$,  $\PP_D\rbr{ \frac12 - \eta(x) \leq t} \leq A t^{\frac{\alpha}{1-\alpha}}$,
where $\eta(x) := \PP_D(y \neq \sign(\inner{w^\star}{x}) \mid x)$ is the label flipping probability on example $x$.
\label{def:tnc}
\end{definition}

Definition~\ref{def:tnc} has two important implications on the data distribution $D$. 
First, setting $t = 0$, we get that $\eta(x) \leq \frac12$ almost surely, which implies that the halfspace $w^*$ is Bayes optimal with respect to $D$.
Second, the fraction of examples $x$ that has a large conditional label flipping probability ($\frac12 - \eta(x) \leq t$) is small (at most $A t^{\frac{\alpha}{1-\alpha}}$). As $A$ decreases and $t$ increases, the noise assumption on $D$ becomes more benign, and the learning problem becomes easier.
Since the initial definition of TNC, the learning theory community has witnessed extensive effort in understanding the necessary and sufficient amount of labels for learning under it, from both statistical and computational perspectives~\citep{hanneke2014theory,hanneke2015minimax,balcan2007margin,balcan2013active,zhang2014beyond,wang2016noise,diakonikolas2020learningb,diakonikolas2020polynomial,zhang2021improved}.
Specialized to the setting of active 
learning halfspaces with TNC under structured unlabeled data distributions:


\begin{itemize}
\item From a statistical perspective, a line of works~\citep{balcan2007margin,balcan2013active,zhang2014beyond,wang2016noise,huang2015efficient} propose algorithms that have a label complexity of $\tilde{O}((\frac1\epsilon)^{2-2\alpha})$, which matches information-theoretic lower bounds~\citep{wang2016noise} in terms of target excess error rate $\epsilon$. However, these algorithms rely on explicit enumeration of classifiers from $\Hcal$ or performing empirical 0-1 loss minimization, which is known to be NP-Hard in general.
\item 
To design a computationally efficient algorithm for active learning halfspaces under Tsybakov noise, a first natural idea is to combine the well-known ``margin-based active learning'' framework~\cite[e.g.][]{balcan2007margin,balcan2013active} with convex surrogate loss minimization. Specifically, we can have an algorithm that iteratively, for each phase $k$: (1) learns a halfspace $w_k$ based on labeled examples $S_k$ using convex surrogate loss minimization; (2) actively collects a new set of labeled examples $S_{k+1}$ in a region close to the decision boundary of $w_k$. Although this algorithm design and analysis framework has made some progress in learning halfspaces under Massart noise~\citep{awasthi2015efficient,awasthi2016learning}, extending it to learning under Tsybakov noise is challenging, in that the Bayes classifier $h_{w^\star}$ can behave arbitrarily poorly (just better than a random guess) in a region with a small probability. 

\item A recent line of pioneering works aim at designing efficient algorithms for passive learning halfspace with Tsybakov noise~\citep{diakonikolas2020learningb, diakonikolas2020polynomial}. Their key insight is that, learning halfspaces can be reduced to the problem of certifying the non-optimality of a candidate halfspace. Using this, ~\cite{diakonikolas2020learningb} developed a quasi-polynomial time learning algorithm with label complexity $d^{O(\frac{1}{\alpha^2} \log^2 (\frac{1}{\epsilon}))}$; and subsequent work~\cite{diakonikolas2020polynomial} designed a polynomial time algorithm with label complexity $(\frac{d}{\epsilon})^{O(\frac{1}{\alpha})}$ under well-behaved distributions, and $\poly(d)\cdot (\frac{1}{\epsilon})^{O(\frac{1}{\alpha^2})}$ under log-concave distributions

\item The first active halfspace learning algorithm for Tsybakov noise that exhibits nontrivial improvements over passive learning is due to~\cite{zhang2021improved}. Their algorithm, based on a nonstandard application of online learning regret inequalities, iteratively optimizes a proximity measure between the iterates and $w^*$. When the Tsybakov noise parameter $\alpha \in (\frac12,1]$, their algorithm has a label complexity of $\tilde{O}( d (\frac1\epsilon)^{\frac{2-2\alpha}{2\alpha-1}} )$. 
\end{itemize}

In summary, for active learning halfspaces with TNC under structured unlabeled data distributions, there still remains a large gap between the label complexity upper bounds achieved by computationally efficient algorithms and the information-theoretic lower bound $\tilde{\Omega}((\frac1\epsilon)^{2-2\alpha})$. 

\paragraph{Our contributions.} In this work, we narrow the above gap by providing an efficient active learning algorithm with a label complexity of $\tilde{O}(d (\frac{1}{\epsilon})^{\frac{8-6\alpha}{3\alpha-1}})$, under the assumption that the noise parameter $\alpha \in (\frac13, 1]$.
In the sample complexity $(\frac{d}{\epsilon})^{O(\frac{1}{\alpha})}$ of the passive algorithm from~\cite{diakonikolas2020polynomial}, the constant hidden in the Big-Oh notation in the exponent is not clear, and this drawback is more significant in terms of the dependence on $d$. On the other hand, our label complexity has a linear dependence on the dimensionality $d$. 
Compared to the first and only efficient active algorithm existing in this setting~\citep{zhang2021improved}, our algorithm expands the feasibility of the noise parameter $\alpha$ from $(\frac12, 1]$ to $(\frac13, 1]$; when $\alpha \in [\frac12, 0.566), (\frac{1}{\epsilon})^{\frac{8-6\alpha}{3\alpha-1}} < (\frac1\epsilon)^{\frac{2-2\alpha}{2\alpha-1}}$. So our algorithm outperforms~\cite{zhang2021improved} when $\alpha \in [\frac13,0.566)$. 
We present the label complexity and computational efficiency of all algorithms in this setting in Table~\ref{tab:existing-results}. 

\begin{table}[t]
    \centering
    \begin{tabular}{c|c|c|c}
        Work & Label complexity upper bound & Passive/Active & Efficient? \\
       ~\cite{balcan2013active}
       & 
       $\tilde{O}(d(\frac1\epsilon)^{2-2\alpha})$
       & 
       Passive & No\\
        ~\cite{diakonikolas2020polynomial}
        &
        $(\frac{d}{\epsilon})^{O(\frac{1}{\alpha})}$
        &
        Passive & Yes \\
        ~\cite{zhang2021improved}
        &
        $\tilde{O}( d (\frac1\epsilon)^{\frac{2-2\alpha}{2\alpha-1}} )$ for $\alpha \in (\frac12, 1]$
        &
        Active & Yes \\
        this work
        &
        $\tilde{O}(d (\frac{1}{\epsilon})^{\frac{8-6\alpha}{3\alpha-1}})$ for $\alpha \in (\frac13, 1]$
        &
        Active & Yes \\
    \end{tabular}
    \caption{A comparison of the state-of-the-art label complexity and efficiency on learning halfspaces with TNC under structured unlabeled data distributions. 
    }
    \label{tab:existing-results}
\end{table}

Our algorithm relies on a few key technical ideas, which we elaborate on below.

\paragraph{Key idea 1: Computationally efficient non-convex optimization for noise tolerance.} The work of~\cite{diakonikolas2020learning} shows that, under the Massart noise condition, optimizing a carefully-chosen non-convex loss $L_\sigma(w) = \EE\sbr{ \ell_\sigma(w, (x,y)) }$ over the noisy labeled data distribution $D$ yield a classifier with low excess error. 
Importantly, ~\cite{diakonikolas2020learning} shows that one does not have to find the global minimum to achieve the above guarantee; instead, finding a first-order stationary point suffices, which admits computationally efficient procedures~\cite[e.g.][]{ghadimi2013stochastic}.
Inspired by this, we show that under Tsybakov noise with $\alpha > \frac13$, for the same nonconvex loss function, a qualitatively-similar structural result holds (Lemma~\ref{lem:gradient-logistic-TNC}). This, when combined with standard results on efficient stochastic optimization methods for finding first-order stationary points~\cite{ghadimi2013stochastic}, yields a passive learning procedure with sample complexity of $T = O( (\frac1\epsilon)^{\frac{8-4\alpha}{3\alpha-1}} )$ that can output a classifier that is close to one of $\cbr{w^*, -w^*}$ with constant probability.

\paragraph{Key idea 2: Label efficient first-order oracle for the non-convex objective.} Our second insight is that, the optimization-based learning algorithm outlined above can be made more label-efficient in our active learning setting. At each iteration of the above algorithm, we call the stochastic gradient oracle of the population loss once. A naive implementation of this oracle requires one labeled example per call: drawing one example $x$ from $D_X$, query the labeling oracle for its label $y$, and return $\nabla \ell_\sigma(w, (x,y))$, the gradient of the loss of the model on example $(x,y)$. Inspired by~\cite{guillory2009active}, we design a much more label-efficient implementation of the stochastic gradient oracle; specifically, each call to the oracle queries $O(\sigma) = O(\epsilon^{\frac{2\alpha}{3\alpha-1}}) \ll 1$ labels in expectation. Moreover, the new implementation of the stochastic gradient oracle preserves the bound on the expected squared norm of the stochastic gradient, resulting in the same iteration complexity $T$. This yields a learning procedure with label complexity of $O(T\sigma) = O( (\frac1\epsilon)^{\frac{8-6\alpha}{3\alpha-1}} )$ that can output a classifier that is close to one of $\cbr{w^*, -w^*}$ with constant probability.

\paragraph{Key idea 3: Label-efficient classifier selection.} The above active learning procedure is yet to achieve the $(\epsilon, \delta)$-PAC learning guarantee, in that: (1) its success probability is constant; (2) if it succeeds, it is possible that its output classifier is close to $-w^*$ as opposed to $w^*$. To address issue (1) and boost the success probability to $1-\delta$, we use a repeat-and-validate procedure similar to~\cite{ghadimi2013stochastic} to obtain multiple independent outputs $\cbr{w_s: s \in [S]}$ one of which is close to $\cbr{w^*, -w^*}$, call the stochastic gradient oracle to estimate $\| \nabla L_\sigma(w_s) \|$ for $s \in [S]$, and choose $\tilde{w}$ to be the $w_s$ with the smallest gradient estimate. Thanks again to the label efficient first-order oracle, this step has a label complexity of $O(d (\frac{1}{\alpha})^{\frac{4-2\alpha}{3\alpha-1}})$. 
To address the issue (2), we observe that under Tsybakov noise, $\tilde{w}$ and $-\tilde{w}$'s error rates differ by a constant; therefore, using a simple 0-1 loss based validation procedure suffices to find a classifier $O(\epsilon)$-close to $w^*$, which has an excess error of $\epsilon$.

\section{Related Work}

\paragraph{Statistical complexity for active learning halfspace under Tsybakov noise condition.}
The statistical complexity for active learning halfspaces under Tsybakov noise condition has been largely characterized over the past two decades~\citep{hanneke2011rates, hanneke2014theory, hanneke2015minimax}. 
For general minimax lower bound of active learning under Tsybakov noise not specific to the hypothesis class of halfspaces,~\citep{hanneke2014theory} provides a minimax label complexity lower bound of $\Omega \rbr{d (\frac{1}{\epsilon} )^{2-2\alpha} }$
~\cite{hanneke2015minimax} establishes minimax label complexity upper and lower bounds for general hypothesis class, in terms of the star number and VC dimension. Specific to the class of homogeneous halfspace, when $\alpha \in (0, \frac12]$, the minimax label complexity has a lower bound of $\Omega \rbr{d (\frac{1}{\epsilon} )^{2-2\alpha} }$; when $\alpha \in [\frac12,1)$, the minimax label complexity has a lower bound of $\Omega \rbr{(\frac{1}{\epsilon} )^{2-2\alpha} (d + (\frac{1}{\epsilon} )^{2\alpha-1})}$.
For a more specific setting for active learning halfspaces under well-behaved distributions,~\citep{wang2016noise} shows a minimax label complexity lower bound of $\Omega \rbr{ (\frac{1}{\epsilon} )^{2-2\alpha} }$. 

On the label complexity upper bound side, assuming the unlabeled distribution is isotropic log-concave,~\cite{balcan2013active} and~\cite{wang2016noise}'s active learning algorithms achieve label complexity of order $\tilde{O} \rbr{d (\frac{1}{\epsilon} )^{2-2\alpha} }$.
These works use a margin-based active learning framework, which is a celebrated algorithmic idea of inductively learning halfspaces under benign unlabeled distribution assumptions. However, these algorithms suffer from computational intractability, since they perform empirical 0-1 risk minimization, which is known to be computationally hard~\citep{arora1997hardness}. 

\paragraph{Efficient passive learning halfspaces under Tsybakov noise condition.}
In spite of the extensive effort for efficiently learning in the presence of Tsybakov noise condition 
~\citep{mammen1999smooth,tsybakov2004optimal}, it has been an outstanding open problem to obtain an efficient learning algorithm for any natural hypothesis class (e.g., parities) until very recent years
, where the first breakthrough is witnessed in passively learning halfspaces under Tsybakov noise condition under well-behaved distributions~\citep{diakonikolas2020polynomial,diakonikolas2020learningb}. Those works adopt the principle of ``reduction from learning to certifying the non-optimality of a candidate halfspace'', and developed quasi-polynomial time
certificate algorithm or polynomial time certificate algorithm, resulting in quasi-polynomial time halfspace learning algorithm (with sample complexity $d^{O(\frac{1}{\alpha^2} \log^2 (\frac{1}{\epsilon}))}$) or polynomial time algorithm (with sample complexity $(\frac{d}{\epsilon})^{O(\frac{1}{\alpha})}$), respectively. 

\paragraph{Efficient active learning halfspaces under Tsybakov noise condition.}
Efficient active learning under Tsybakov noise condition is conceptually more difficult than efficient passive learning. The difficulty largely lies in the ''conflict'' between the nature of Tsybakov noise condition - it allows even the Bayes classifier $w^\star$ to 
have an error rate arbitrarily close to $1/2$
in a region with a small enough probability - and the common analysis technique adopted in active learning. In more detail, the combination of localized sampling and iterative convex surrogate loss minimization technique used in many efficient active learning algorithms is hard to analyze in this setting, as they oftentimes require learning model with a small constant error rate in localized regions, which is hard to establish under Tsybakov noise.
Many efficient active learning algorithms~\citep{awasthi2014power,awasthi2015efficient,awasthi2016learning} adopt the idea of localization to some extent. To overcome this barrier and obtain an efficient active algorithm to learn halfspace, substantially novel algorithmic ideas seem to be necessary.  

In this regard, there are even fewer works along the direction of active learning under Tsybakov noise. One notable work is~\cite{zhang2021improved}, where an active learning algorithm is developed for learning halfspaces under $(A, \alpha)-$Tsybakov noise condition for $\alpha \in (\frac12,1]$ and well-behaved distribution, and achieves a label complexity of $\tilde{O}(d (\frac{1}{\epsilon})^{\frac{2-2\alpha}{2\alpha-1}})$. 

Due to space constraints, we defer the discussions of additional related works in Appendix~\ref{sec:addl-rel-works}.

\section{Preliminaries}
\label{sec:prelims}

A (homogenous) halfspace, or a linear classifier, is a function $h_w: \RR^d \mapsto \cbr{ \pm 1}$ that is defined as $h_w(x) = \sign(\inner{w}{x})$, where $w \in \RR^d$.   
In this paper,
we consider the standard binary classification setting, where the hypothesis class $\Hcal$ is the set of halfspaces $\cbr{h_w: w \in \RR^d}$. 
In the sequel, to ease the notation, we frequently use $w$ to represent the halfspace $h_w$ defined by the vector $w \in \RR^d$.  
We denote by $D$ the joint distribution of
labeled examples $(x,y)$ supported on $\RR^d \times \cbr{\pm 1}$ and denote by $D_X$ the marginal distribution of $D$ on $x$. 
We define the empirical error rate of $h$ on $S$, $\err_S(h) := \frac1{|S|} \sum_{(x,y) \in S} \ind\rbr{h(x) \neq y}$.  
For $N \in \NN_+$, let $[N] := \cbr{1, 2, \ldots, N}$. Throughout this paper, for $a,b \in \RR^d$, we use $\|a\|$ to denote $\|a\|_2$, the $\ell_2$ norm of $a$, and use $\inner{a}{b}$ to denote the inner product of $a$ and $b$, and denote by $\theta(a,b) = \arccos \rbr{\frac{\inner{a}{b}}{\|a\|\|b\|}} \in [0,\pi]$ the angle between them.

Following the distributional assumptions in~\citep{diakonikolas2020polynomial, diakonikolas2020learningb, zhang2021improved}, this work proceeds in developing an efficient algorithm for active learning halfspace under Tsybakov noise condition. We assume $D_X$, the marginal distribution over the instance space, lies in the family of well-behaved distributions, which generalizes the isotropic log-concave distribution~\citep{balcan2013active, awasthi2014power, awasthi2016learning} and the uniform distribution on the $d$-dimensional unit sphere~\citep{awasthi2015efficient,yan2017revisiting,wang2016noise}. We formally define well-behaved distributions as follows:

\begin{definition}[Well-behaved distributions~\citep{diakonikolas2020polynomial}]
\label{def:well-behaved}
Fix  $L, R, U, \beta > 0$. We say a distribution $D_X$ on $\RR^d$ to be $(2,L,R, U, \beta)$ well-behaved, or well-behaved for short, if the following properties are satisfied:
for all $x$ randomly drawn from $D_X$,
let $x_V$ be the projected coordinates of $x$ onto any $2$-dimensional linear subspace $V$ of $\RR^d$, and $p_V$ be the corresponding probability density function on $\RR^2$. $p_V$ satisfies, 
\begin{enumerate}
    \item $p_V(z) \geq L$, for all $z$ such that $\| z \|_2 \leq R$;
    \item $p_V(z) \leq U$, for all $z \in \RR^2$;  
\end{enumerate}
moreover, for any unit vector $w$ in $\RR^d$ and any $t > 0$, $\PP_{D_X}(\abs{\inner{w}{x}} \geq t) \leq \exp(1-\frac{t}{\beta})$.
\end{definition}

An important property of the objective function in iterative optimization is the smoothness property, which we define below: 

\begin{definition}
A twice continuous differentiable function $F$ is $L$-smooth on $\Dcal$, 
if $\|\nabla^2 F(x) \|_{\mathrm{op}} \leq L$, for all $x \in \Dcal$.  
\label{def:smoothness}
\end{definition}

\section{Algorithm}

\begin{algorithm}
\caption{Active learning halfspaces under TNC}
\label{alg:main}
\begin{algorithmic}[1]
\STATE {\bfseries Input:} Target excess error $\epsilon$, failure probability $\delta$

\STATE $\theta_0 \gets O \rbr{ \frac{1}{\ln^2 \frac{1}{\epsilon}} \frac{\epsilon}{2} }$, $\sigma \gets \Theta \rbr{(\frac{1}{A})^\frac{1-\alpha}{3\alpha-1}\theta_0^{\frac{2\alpha}{3\alpha-1}} }$, $\rho \gets \Theta \rbr{ (\frac{1}{A})^{\frac{2(1-\alpha)}{3\alpha-1}} \theta_0^{\frac{2(1-\alpha)}{3\alpha-1}} }$, 
$S \gets \log \frac{6}{ \delta}$
\label{line:parameter-set-up}

\FOR{$s =1,2,\ldots,S$}
\label{line:rep-start}
\STATE $w_s \gets \apsgd(N = \tilde{O}(\frac{d}{\sigma^2 \rho^4}), \beta = \tilde{\Theta}(\frac{ \rho^2 \sigma^2}{d}))$
(see Algorithm~\ref{alg:active-PSGD-finding-stationary-point})

\label{line:iterative-optimization}
\ENDFOR
\FOR{$s =1,2,\ldots,S$}
\label{line:selection-1-start}
\STATE $g_{s,1}, \ldots, g_{s,M_1} \gets \afo(w_s)$ (see Algorithm~\ref{alg:first-order-oracle})
\label{line:selection-1}
\STATE $\bar{g}_s \gets \frac{1}{M_1} \sum_{i=1}^{M_1} g_{s,i}$
\ENDFOR
\STATE $s^* \gets \argmin_{s \in [S]} \|\bar{g}_s\|$
\STATE $\tilde{w} \gets w_{s^*}$
\label{line:selection-1-end}

\STATE Draw $M_2$ unlabeled examples from $D_X$ and query their labels labeled samples $\cbr{(x_i, y_i)}_{i=1}^{M_2}$
\label{line:selection-2-start}

\STATE $\hat{w} \gets \argmin_{w \in \cbr{\pm \tilde{w} }} \frac{1}{M_2} \sum_{i=1}^{M_2} \ind (\sign(\inner{w}{x_i}) \neq y_i)$ 
\label{line:selection-2-end}

\STATE {\bfseries Return:} $\hat{w}$

\end{algorithmic}
\end{algorithm}

Our main algorithm (Algorithm~\ref{alg:main}) consists of three key components, namely (1) iterative non-convex optimization with active label queries (\apsgd, Algorithm~\ref{alg:active-PSGD-finding-stationary-point}); (2) label-efficient iterate selection to boost the success probability (lines~\ref{line:rep-start} to~\ref{line:selection-1-end}); (3) label-efficient final iterate selection (lines~\ref{line:selection-2-start} to~\ref{line:selection-2-end}). We now discuss each component in more detail. 

\subsection{Efficient non-convex optimization with active label queries (Algorithm~\ref{alg:active-PSGD-finding-stationary-point})}

\begin{algorithm}[t]
\caption{\apsgd: Projected SGD for finding a stationary point of $L_\sigma$ using active learning}
\label{alg:active-PSGD-finding-stationary-point}
\begin{algorithmic}[1]
\STATE {\bfseries Input:} number of steps $N$, step size $\beta$

\STATE $w_0 \gets e_1$
\FOR{$i=1,2,\ldots,N$}
\STATE $g_i \gets \afo(w_{i-1})$
\STATE $v_i \gets w_{i-1} - \beta g_i$
\STATE $w_i \gets \frac{v_i}{\|v_i\|_2}$
\label{line:normalization}

\ENDFOR

\STATE {\bfseries Return:} $w_R$, where $R$ is a random variable uniformly distributed over $\cbr{0,\ldots,N-1}$

\end{algorithmic}
\end{algorithm}

As we will see in the Appendix~\ref{sec:addl-rel-works}, 
there are several results~\citep{awasthi2015efficient,awasthi2016learning,diakonikolas2019distribution} in the literature against the one-shot application of convex surrogate loss minimization in efficient learning halfspaces with noise. One possible way to get around this is to instead adopt a non-convex surrogate loss. 
We denote by $\phi_\sigma(t) := \frac{1}{1+ e^{\frac{t}{\sigma}}}$ the softmax loss function, first proposed by~\cite{diakonikolas2020learning}, which can be viewed as a smooth approximation of 0-1 loss. For a halfspace $w$, we let $L_\sigma(w) := \EE_{(x,y) \sim D} \phi_\sigma \rbr{y \frac{\inner{w}{x}}{\|w\|}}$ be its normalized expected softmax loss function.
Our key observation is that, in the presence of Tsybakov noise, to find a $w$ close to $w^*$, it suffices to find an approximate first-order stationary point of the softmax loss $L_\sigma$. The technique of using the softmax loss in the optimization procedure and proving an approximate stationary point suffices for the halfspace learning goal is originally developed in~\citep{diakonikolas2020learning}, where it provides an efficient passive learning algorithm for learning halfspaces under Massart noise. In this work, we extend this technique to the setting of learning halfspaces under Tsybakov noise. 
Formally, we prove:

\begin{lemma}
Let $D_X$ be a well behaved distribution, and $D$ satisfies $(A, \alpha)$-TNC. Denote by $L_\sigma(w) = \EE_D \sbr{\phi_\sigma \rbr{y \frac{\inner{w}{x} }{\|w\|_2}}}$ where $\phi_\sigma$ is softmax loss 
defined above. Let $w$ be such that $\theta(w, w^*) \in (\theta, \pi - \theta)$, where $\theta \leq \Theta(A)$.
Then for $\sigma =  \Theta \rbr{\theta^\frac{2\alpha}{3\alpha-1}} $, we have that 
$\| \nabla_w L_\sigma(w) \|_2
\geq 
\Omega\rbr{\theta^\frac{2(1-\alpha)}{3\alpha-1} } := 2 \rho$. 
\label{lem:gradient-logistic-TNC}
\end{lemma}
Lemma~\ref{lem:gradient-logistic-TNC} establishes the connection between 0-1 loss and the $\ell_2$ norm of the gradient of a carefully designed non-convex loss function - softmax loss $L_\sigma(w) = \EE_D \sbr{\phi_\sigma \rbr{y \frac{\inner{w}{x} }{\|w\|_2}}}$. 
Specifically, the parameter $\rho$ and $\sigma$ are delicately chosen in a way such that if $\|\nabla L_\sigma(w) \| \leq 2\rho$, then 
either $w$ or $-w$
is at an angle at most $\theta_0 := O \rbr{\frac{1}{\ln^2 \frac{1}{\epsilon}} \frac{\epsilon}{2} }$ from $w^*$. By Lemma~\ref{lem:prob-angle}, either $w$ or $-w$ has
an excess error at most $\epsilon$ as desired.

To efficiently find 
an approximate first-order stationary point of $L_\sigma$, 
we adapt the iterative procedure of randomized stochastic gradient (RSG) in~\cite{ghadimi2013stochastic} to this setting, resulting in Alg.~\ref{alg:active-PSGD-finding-stationary-point}. 
More precisely, 
\apsgd (Algorithm~\ref{alg:active-PSGD-finding-stationary-point}) 
aims at iteratively
obtaining a halfspace $w$ such that $\|\nabla L_\sigma(w) \| \leq \rho$ with a constant probability.

In more detail, 
\apsgd takes as input the number of steps $N$ and a constant stepsize $\beta$. 
$w_1$ is initialized randomly on the unit $\ell_2$-ball in $\RR^d$.
In each iteration $i$, \apsgd 
calls function \afo (Algorithm~\ref{alg:first-order-oracle}), which serves as a first-order stochastic gradient oracle for $L_\sigma$ to obtain 
{$g_i$, an unbiased estimate of $\nabla L_\sigma(w_i)$}, and updates the previous iterate $w_{i-1}$ with the step size $\beta$. As we will see, from item~\ref{item:perp} of Lemma~\ref{lem:active-oracle-property} that the direction of the stochastic gradient estimate $g_i$ is always perpendicular to the previous iterate $w_{i-1}$, hence all $v_i$'s satisfy $\|v_i\|_2 \geq 1$. The next step (line~\ref{line:normalization}) is to project $v_i$ back to the unit $\ell_2$-ball to obtain $w_i$. 
Lastly, after $N$ iterations, Algorithm~\ref{alg:active-PSGD-finding-stationary-point} output one iterate from $\cbr{w_i: i \in \cbr{0, \ldots, N-1}}$ uniformly at random.
We have the following performance guarantee of \apsgd:

\begin{lemma}
Let $\rho, \sigma$ be defined as in line~\ref{line:parameter-set-up} of Algorithm~\ref{alg:main}.
If Algorithm~\ref{alg:active-PSGD-finding-stationary-point} receives inputs
$N = \tilde{O}(\frac{d}{\sigma^2 \rho^4})$, $\beta = \tilde{\Theta}(\frac{ \rho^2 \sigma^2}{d})$, 
then its output $w_R$ of Algorithm~\ref{alg:active-PSGD-finding-stationary-point} satisfies, with probability at least $\frac12$, \[
    \|\nabla L_\sigma(w_R) \| 
    \leq
    \rho
    \]
Furthermore, during $N$ iterations, with probability at least $1-\frac{\delta}{6S}$ the total number of label queries is at most 
$T_1 :=
\tilde{O} \rbr{
 d (\frac{1}{\epsilon})^\frac{8-6\alpha}{3\alpha-1}  } 
 $

\label{lem:active-SGD}
\end{lemma}

\begin{remark}
\cite{arjevani2022lower} shows that under the assumption of smooth objective and bounded expected squared norm of the stochastic gradient, RSG achieves the optimal first-order oracle complexity.
Thus, we speculate that the iteration complexity in Lemma~\ref{lem:active-SGD} cannot be improved significantly using some other algorithm. 
\end{remark}

    The idea of \apsgd bears similarity with the standard iterative optimization method, with some remarkable innovation. 
    The key insight of stochastic gradient descent is that by obtaining an unbiased stochastic gradient, each iteration is making progress toward achieving the optimization goal in expectation. 
    In the passive learning setup, the typical way of 
    implementing the stochastic gradient oracle is to 
    sample $(x,y)$ from the labeled data distribution $D$.

    We show that in our active learning setup, the stochastic gradient oracle can be implemented more label-efficiently by the \afo procedure (Algorithm~\ref{alg:first-order-oracle}). In $\afo$,
    we carefully design a function of query probability $q(w,x) := \sigma \abs{\phi_\sigma'(\inner{\frac{w}{\|w\|}}{x} ) }$ - not all unlabeled example $x$ are equally important or equally informative for our optimization purpose. Intuitively, the closer the $x$ lies to the decision boundary of halfspace $w$, the more informative it is (and as a consequence, we query the label for this $x$ with higher probability), because the current $w$ is less confident in labeling $x$ - this idea coincides with the renowned margin-based method~\citep{balcan2013active,wang2016noise,awasthi2015efficient,awasthi2016learning} in the active learning literature.

\begin{algorithm}[t]
\caption{\afo: stochastic gradient oracle for $L_\sigma$ exploiting active learning
}
\label{alg:first-order-oracle}
\begin{algorithmic}[1]
\STATE {\bfseries Input:}{Unit vector $w$}

\STATE Sample $x$ from $D_X$
\STATE Draw $Z \sim \Ber(q(w,x))$, where the query probability $q(w,x) := \sigma \abs{\phi_\sigma'(\inner{\frac{w}{\|w\|}}{x} ) }$
\IF{$Z = 1$}
\STATE $y \gets$ query the labeling oracle on example $x$
\STATE {\bfseries Return:} $h(w,x,y) :=
        - \frac{1}{\sigma} y \rbr{\frac{x}{\|w\|_2} - \frac{\inner{w}{x} w}{\|w\|_2^3} } $
\ELSE 
\STATE {\bfseries Return:} $0$
\ENDIF
\end{algorithmic}
\end{algorithm}
Whenever \afo is invoked with an input unit vector $w$, it firstly draws an unlabeled example $x$ from $D_X$, and computes the label query probability on this $x$ according to $q(w,x)
$. Note that $q(w,x)$ is a valid probability, i.e., $0 \leq q(w,x) \leq 1$ for all $w,x \in \RR^d$, since $\abs{\phi_\sigma'(t)} \leq \frac{1}{\sigma}$ for all $t \in \RR$. Then it draws a Bernoulli random variable $Z$ with success probability $q(w,x)$. If $Z = 1$, then \afo outputs the vector $h(w,x,y) :=
        - \frac{1}{\sigma} y \rbr{\frac{x}{\|w\|_2} - \frac{\inner{w}{x} w}{\|w\|_2^3} } $, otherwise, it outputs a zero vector.

\begin{remark}
We show in item~\ref{item:expected-label} of Lemma~\ref{lem:active-oracle-property} that \afo is label-efficient; furthermore, 
although \afo only queries labels for a fraction of unlabeled examples $x$ it happens to sample, \afo preserves the bound on the expected squared norm of the stochastic gradient (see Claim~\ref{claim:passive-variance}), resulting in the same iteration complexity $N$ as passively querying the labels for all $x$. 
\end{remark}

\begin{remark}
\afo is inspired by~\cite{guillory2009active}, where it provides sampling rules \textbf{Query}$(w,x)$ and update rules \textbf{Update}$(w,x,y)$ for several commonly used margin-based losses. While this work exhibits some experimental results, it does not provide a theoretical analysis of this query strategy.  
\end{remark}

Despite being simple the oracle behavior at first sight, \afo enjoys several properties that turn out to be essential in guaranteeing the desirable performance in our main algorithm, Algorithm~\ref{alg:main}. 
We present Lemma~\ref{lem:active-oracle-property} for the delicate properties of \afo. 
\begin{lemma}
    Let $g_w$ be the random output of $\afo(w)$. We have, for any unit vector $w$:
    \begin{enumerate}
        \item $g_w$ is perpendicular to $w$;
        \label{item:perp}
        \item $g_w$ is an unbiased estimator of $\nabla L_\sigma(w): \EE \sbr{g_w}= \nabla L_\sigma(w)$;
        \label{item:unbiased}
        \item $\EE \sbr{\|g_w \|^2} \leq \tilde{O} (\frac{d}{\sigma})$;
        \label{item:variance}
        \item The expected number of label queries per call to \afo is $\tilde{O}(\sigma)$.
        \label{item:expected-label}
    \end{enumerate}
\label{lem:active-oracle-property}

\end{lemma}

Furthermore, as we will see in the next subsection, \afo is not only utilized in the iterative non-convex optimization procedure \apsgd, but also in the iterate selection, both of which help reduce the label complexity of the overall algorithm. 

\subsection{Label-efficient iterate selection to boost the success probability (lines~\ref{line:rep-start} to~\ref{line:selection-1-end}) }

Recall that \apsgd only guarantees that $\| \nabla L_\sigma(w) \| \leq \rho$ with a constant probability. To achieve the $(\epsilon,\delta)$-PAC learning goal, lines~\ref{line:rep-start} to~\ref{line:selection-1-end} in Algorithm~\ref{alg:main} boost the above success probability to $1-O(\delta)$.

At a high level, our method 
follows the classic trick of re-running an algorithm for multiple independent trials and picking the best output. 
One naive idea to pick the best output, is to sample a set of validation examples from $D$ and pick the $w$ in $\cbr{w_1, -w_1, \ldots, w_S, -w_S}$ that has the lowest validation error. An application of Hoeffding's inequality shows that, setting the validation sample size to $\tilde{O} \rbr{\frac{1}{\epsilon^2}}$ suffices to find a desired halfspace whose excess error at most $\epsilon$. Together with the labeling cost in the iterative non-convex optimization, it yields a total label complexity of 
$\tilde{O} \rbr{d (\frac{1}{\epsilon})^\frac{8-6\alpha}{3\alpha-1}  
+ \frac{1}{\epsilon^2}}$, 
which is substantially suboptimal to our current label complexity $\tilde{O} \rbr{
 d (\frac{1}{\epsilon})^\frac{8-6\alpha}{3\alpha-1} }$ in Theorem~\ref{thm:main}, when $\alpha \geq \frac{5}{6}$. 
Here, we design a specialized procedure that achieves better label efficiency by re-utilizing our label-efficient first-order stochastic gradient oracle \afo. 

The idea of conducting the iterate selection by the gradient norm instead of the validation error is largely inspired by the two-phase RSG (2-RSG) in~\cite{ghadimi2013stochastic}, where the analysis on the total number of first-order oracle calls is under the assumption of sub-gaussian stochastic gradient. We show that the output of \afo is sub-exponential (Lemma~\ref{lem:sub-exponential}) and re-analyze the oracle complexity.

We re-run Algorithm~\ref{alg:active-PSGD-finding-stationary-point} independently for $S$ times, which ensures that the probability that no $w$ in $\cbr{w_s: s \in [S]}$ has $\|\nabla L_\sigma(w) \| \leq \rho$ is $2^{-\Theta(S)}$.  
The selection step using the first-order stochastic gradient oracle \afo (lines~\ref{line:selection-1-start} to~\ref{line:selection-1-end})
is done as follows. After we obtain one halfspace candidate $w_s$ in each iteration, we call \afo $M_1$ times and take the average of all outputs, to obtain a good estimate $\bar{g}_s$ of the gradient of $L_\sigma(w_s)$; therefore, $\| \bar{g}_s \|$ closely approximates $\nabla L_\sigma(w_s)$.
After we collect gradient estimates for all $S$ candidate halfspaces, we pick the one with the smallest gradient norm estimate $\| \bar{g}_s \|$. We show that our label-efficient iterate selection procedure (lines~\ref{line:rep-start} to~\ref{line:selection-1-end}) enjoys the following performance guarantee: 

\begin{lemma}
Let $\rho, \sigma$ be defined as in line~\ref{line:parameter-set-up} of Algorithm~\ref{alg:main}. Suppose $w_1, \ldots, w_S$ are such that $\min_i \|\nabla L_\sigma (w_i) \| \leq \rho$, then after executing lines~\ref{line:selection-1-start} to~\ref{line:selection-1-end} of Algorithm~\ref{alg:main}, with \[
    M_1 =
    c \frac{d}{\sigma^2 \rho^2} \ln\frac{S}{\delta}
    \]
    for some constant $c$,
    with probability at least $1-\delta/6$, $\tilde{w}$ satisfies \[
     \|\nabla L_\sigma (\tilde{w}) \| \leq 2\rho. 
     \]
     Furthermore, after $M_1$ calls to \afo, with probability at least $1-\frac{\delta}{6S}$, the total number of label queries is at most
     $T_2 := \tilde{O} (d (\frac{1}{\epsilon})^\frac{4-2\alpha}{3\alpha-1} )$. 
     \label{lem:selection-1}
\end{lemma}

Lemma~\ref{lem:selection-1} shows that, if there exists $w$ in $\cbr{w_s: s \in [S]}$ such that $\|\nabla L_\sigma(w) \| \leq \rho$ (which is true with high probability after we re-run \apsgd independently for $S$ times), then after executing lines~\ref{line:selection-1-start} to~\ref{line:selection-1-end}, we have, with high probability, $\|\nabla L_\sigma (\tilde{w}) \| \leq 2\rho$. 

To achieve label efficiency in the selection procedure, we prove in Lemma~\ref{lem:sub-exponential} that the stochastic gradients output by \afo are sub-exponential, and we apply a large-deviation bound of vector-valued sub-exponential random variables 
(Lemma~\ref{lem:sub-exp-concentration}, which is Theorem 2.1 in~\cite{juditsky2008large})
to prove the performance guarantee of this iterate selection in Lemma~\ref{lem:selection-1}. 
In this way, the labeling cost in the iterate selection step is of lower order than that in the iterative non-convex optimization.

\subsection{Label-efficient final iterate selection (lines~\ref{line:selection-2-start} to~\ref{line:selection-2-end})}

Up to now, combining Lemmas~\ref{lem:active-SGD} and~\ref{lem:selection-1}, we have successfully shown, with high probability, 
either $\tilde{w}$ or $-\tilde{w}$ has an excess error at most $\epsilon$. Our last task is to pick the right one out of the pair. To this end, we draw $M_2 = \tilde{O}(1)$ iid labeled examples from $D$, and pick the one from $\cbr{ \pm \tilde{w}}$ that has a lower empirical error on this sample set. 
We have the following lemma on the performance guarantee on the final iterate selection phase:

\begin{lemma}
    Suppose $\tilde{w}$ satisfies that $\exists w \in \cbr{\pm \tilde{w} }$, such that $\err(w) - \err(w^*) \leq \epsilon$ with $\epsilon \leq \frac12 \alpha (\frac{1}{A})^\frac{1-\alpha}{\alpha}$, then after executing lines~\ref{line:selection-2-start} to~\ref{line:selection-2-end} of Algorithm~\ref{alg:main}, where
    $M_2 =  O\rbr{ A^\frac{2-2\alpha}{\alpha} \frac{1}{\alpha^2} \ln \frac{1}{\delta} } $
    , we have that
    with probability at least $1-\delta/3$, $\hat{w}$ satisfies \[
    \err(\hat{w}) - \err(w^*) \leq \epsilon
    \]
    \label{lem:selection-2}
\end{lemma}

Lemma~\ref{lem:selection-2} shows that, if the target error $\epsilon$ satisfies $\epsilon \leq \frac12 \alpha (\frac{1}{A})^\frac{1-\alpha}{\alpha}$ (a constant), then a constant number $M_2$ of labeled examples suffice to 
find, with high probability, the one with desired excess error guarantee.

\section{Performance Guarantees}

\begin{theorem}
Suppose $D$ satisfies $(A, \alpha)$-Tsybakov noise condition with $\alpha  \in (\frac{1}{3}, 1]$ and the marginal distribution $D_X$ is well-behaved.
For any $\epsilon \leq \min( \tilde{\Theta}(A), \frac12 \alpha (\frac{1}{A})^\frac{1-\alpha}{\alpha})$, and $\delta \in (0,1)$, with probability at least $1 - \delta$, Algorithm~\ref{alg:main} outputs a halfspace $\hat{w}$, such that $\err(\hat{w}) - \err(w^*) \leq \epsilon$. 
In addition, 
its total number of label queries is at most 
$\tilde{O} \rbr{ 
d A^\frac{9-9\alpha}{3\alpha-1} (\frac{1}{\epsilon})^\frac{8-6\alpha}{3\alpha-1}
}   $.
\label{thm:main}
\end{theorem}
We compare this label complexity to the state-of-the-art sample/label complexities of passive learning~\citep{diakonikolas2020polynomial} and active learning~\citep{zhang2021improved} existing in the literature. Our work achieves a linear dependency on the dimensionality $d$, and an explicit exponent on the target error $\epsilon$; however, in the sample complexity $(\frac{d}{\epsilon})^{O(\frac{1}{\alpha})}$ of the passive algorithm from~\cite{diakonikolas2020polynomial}, the constant hidden in the Big-Oh notation in the exponent is not clear. Compared to the first and only efficient active algorithm existing in this setting~\citep{zhang2021improved}, our algorithm expands the feasibility of the noise parameter $\alpha$ from $(\frac12, 1]$ to $(\frac13, 1]$; when $\alpha \in [\frac12, 0.566), (\frac{1}{\epsilon})^{\frac{8-6\alpha}{3\alpha-1}} < (\frac1\epsilon)^{\frac{2-2\alpha}{2\alpha-1}}$. So our algorithm outperforms~\cite{zhang2021improved} when $\alpha \in [\frac13,0.566)$.

\begin{proof}
Recall that at the beginning of Algorithm~\ref{alg:main}, we set the parameters $\sigma = \Theta \rbr{\theta_0^{\frac{2\alpha}{3\alpha-1}} }$, $\rho = \Theta \rbr{ \theta_0^{\frac{2(1-\alpha)}{3\alpha-1}} }$, where $\theta_0 = O \rbr{\frac{1}{\ln^2 \frac{1}{\epsilon}} \frac{\epsilon}{2} }$.
Given our assumption that $\epsilon \leq \tilde{\Theta}(A)$, we have $\theta_0 \leq \Theta(A)$.
Also, 
recall from Lemma~\ref{lem:active-SGD} and Lemma~\ref{lem:selection-1} that $T_1 = 
\tilde{O} \rbr{
 d (\frac{1}{\epsilon})^\frac{8-6\alpha}{3\alpha-1}  } 
$, $T_2 = 
\tilde{O} (d (\frac{1}{\epsilon})^\frac{4-2\alpha}{3\alpha-1} )
$. 
We define the following events of interest,   
\[
E_1 = \cbr{ \min_{s \in [S]} \|\nabla L_\sigma (w_s) \| \leq \rho 
\wedge 
\text{the total number of label queries at line~\ref{line:iterative-optimization} is at most $S \cdot T_1$}
}
\]
\[
E_2 = \cbr{ \min_{s \in [S]} \|\nabla L_\sigma (w_s) \| \leq \rho \implies 
 \|\nabla L_\sigma (\tilde{w}) \| \leq 2\rho
\wedge
\text{the total number of label queries at line~\ref{line:selection-1} is at most $S \cdot T_2$}
}
\]
\[
 E_3 = \cbr{
\exists w \in \cbr{\pm \tilde{w} } \text{ s.t. } \err(w)- \err(w^*) \leq \epsilon 
 \implies \err(\hat{w}) - \err(w^*) \leq \epsilon \wedge  \text{line~\ref{line:selection-2-start} queries $M_2$ labels} }
\]

By Lemma~\ref{lem:active-SGD}, 
for each $s \in [S]$, with probability at least $\frac12$, $
    \|\nabla L_\sigma(w_s) \| 
    \leq
    \rho
    $.
Furthermore, during $N$ iterations, with probability at least $1-\frac{\delta}{6S}$ the total number of label queries is at most $T_1 
$. 
Since Algorithm~\ref{alg:active-PSGD-finding-stationary-point} is executed for $S = \log \frac{6}{ \delta}$ times, and each run is independent, we have that with probability at least $1-\frac{\delta}{6}$, $\min_{s \in [S]} \|\nabla L_\sigma (w_s) \| \leq \rho$. 
Applying a union bound on the total number of label queries in $S$ runs of Algorithm~\ref{alg:active-PSGD-finding-stationary-point}, we have that with probability at least $1-\frac{\delta}{6}$, the total number of label queries at line~\ref{line:iterative-optimization} is at most $S \cdot T_1$. Applying again a union bound, we have $\PP(E_1) \geq 1-\delta/3$. 

By Lemma~\ref{lem:selection-1}, together with a union bound on the total number of label queries in $S$ iterations of line~\ref{line:selection-1}, we have $\PP(E_2) \geq 1-\delta/3$.
By Lemma~\ref{lem:selection-2}, $\PP(E_3) \geq 1-\delta/3$. Define $E = E_1 \cap E_2 \cap E_3$. By union bound, $\PP(E) \geq 1-\delta$. For the rest of the proof, we condition on event $E$ happening.

Since both $E_1$ and $E_2$ happen, $\|\nabla L_\sigma (\tilde{w}) \| \leq 2\rho$.
Taking the contrapositive of Lemma~\ref{lem:gradient-logistic-TNC}, we have that if $ \|\nabla L_\sigma (\tilde{w}) \| \leq 2\rho$, then
$\min \cbr{\theta(\tilde{w}, w^*), \theta(-\tilde{w}, w^*)} \leq \theta_0$.

Applying Lemma~\ref{lem:prob-angle} with $\gamma = \frac{\epsilon}{2}$, we have that if a halfspace $w$  
satisfies $\theta(w, w^*) \leq \theta_0$, then $\PP_{x \sim D_X}(h_w(x) \neq h_{w^*}(x)) \leq \epsilon$, which, in turn, implies
such that $\err(w) - \err(w^*) \leq \epsilon$. Hence $\exists w \in \cbr{\pm \tilde{w} } \text{ s.t. } \err(w)- \err(w^*) \leq \epsilon $. 
By the definition of $E_3$, $\err(\hat{w}) - \err(w^*) \leq \epsilon$. 
Therefore, we conclude that  with probability at least $1 - \delta$, the final output $\hat{w}$ of Algorithm~\ref{alg:main} satisfies $
    \err(\hat{w}) - \err(w^*) \leq \epsilon
    $. 

The total label complexity of Algorithm~\ref{alg:main} is at most 
\begin{align*}
 S \cdot T_1 + S \cdot T_2 + M_2 
 =&
S \cdot \tilde{O} \rbr{
 d (\frac{1}{\epsilon})^\frac{8-6\alpha}{3\alpha-1}  } 
 + S \cdot \tilde{O} \rbr{d (\frac{1}{\epsilon})^\frac{4-2\alpha}{3\alpha-1} } + O(\ln \frac{6}{\delta}) \\
=&
\tilde{O} \rbr{ 
d (\frac{1}{\epsilon})^\frac{8-6\alpha}{3\alpha-1}
}   
\end{align*}
\end{proof}

\section{Conclusions and open problems}

In this work, we provide a computationally and label efficient active learning algorithm that succeeds in learning a halfspace under $(A, \alpha)$-Tsybakov noise condition under well-behaved unlabeled distributions. Our algorithm achieves a label complexity of $\tilde{O}(d (\frac{1}{\epsilon})^{\frac{8-6\alpha}{3\alpha-1}})$, under the assumption that the noise parameter $\alpha \in (\frac13, 1]$.

While our algorithm narrows down the gap between the label complexities of the previously known passive or active efficient algorithms~\citep{diakonikolas2020polynomial,zhang2021improved} and the information-theoretic lower bound, it remains an outstanding open problem to obtain an efficient active learning algorithm 
that can match the label complexity of inefficient active algorithms $\tilde{O} (d (\frac{1}{\epsilon} )^{2-2\alpha})$ or information-theoretic lower bound $\tilde{\Omega} ( (\frac{1}{\epsilon} )^{2-2\alpha} )$ for all $\alpha \in (0,1]$. 

\bibliographystyle{plainnat}
\bibliography{learning}

\appendix

\section{Additional Related Work}
\label{sec:addl-rel-works}

\paragraph{Efficient learning halfspaces under benign noise. }
Besides Tsybakov noise condition, 
several other benign noise models have been proposed and studied in the learning theory literature. Among these, the simplest one is Random Classification Noise (RCN)~\citep{angluin1988learning}, where at each $x$, the label is flipped independently with the same probability. It is known that halfspaces are efficiently learnable under RCN~\citep{blum1998polynomial}. 

In addition to RCN, several more realistic noise models are developed and studied, with the most distinguished one being the Massart noise condition, where at each unlabeled datapoint $x$, the label flipping probability is \emph{at most} $\eta$. Several eminent works in learning theory literature are dedicated to developing efficient learning halfspace algorithms under Massart noise condition~\citep{awasthi2015efficient,awasthi2016learning,yan2017revisiting,zhang2018efficient,zhang2020efficient,diakonikolas2019distribution,diakonikolas2020learning,zhang2021improved}. 
To name a few,~\cite{awasthi2015efficient} is among the earliest works in this thread, where an efficient algorithm is developed under the assumption that the unlabeled distribution is uniform distribution. However, this analysis is subject to the restriction that the Massart noise parameter is such that $\eta < 3 \times 10^{-6}$. Since then, subsequent works have made major improvements in label complexity in efficient learning halfspaces under Massart noise.~\cite{zhang2020efficient} develops an efficient active learning algorithm with a label complexity of $O\rbr{ \frac{d}{(1-2\eta)^4} \polylog(\frac1\epsilon) }$, assuming the unlabeled data distribution is a log-concave distribution. Finally, the label complexity gap compared to the information-theoretic result is closed in~\cite{zhang2021improved}, whose algorithm achieves a label complexity of $O\rbr{\frac{d}{(1-2\eta)^2} \polylog(\frac1\epsilon)}$ under the assumption of well-behaved unlabeled distribution. 

Besides distribution-specific setting under Massart noise condition, see more in~\cite{awasthi2016learning,yan2017revisiting,diakonikolas2020learning}, there are recent breakthroughs in distribution-free setting~\citep{diakonikolas2019distribution,chen2020classification}. 
Specifically,~\cite{diakonikolas2019distribution} provides an efficient algorithm that can improperly learn a halfspace with a misclassification error guarantee $\eta + \epsilon$ in $\poly(d, \frac{1}{\epsilon})$ time. This work addresses a long-standing open problem of whether there exists a distribution-free weak learner in the presence of Massart noise.~\cite{chen2020classification} strengthens this result by providing an efficient proper halfspace learning algorithm that can achieve the same misclassification error guarantee, with an improved bound on the sample complexity. It also provides a black-box ``distillation'' procedure that converts any classifier to a proper halfspace without losing prediction accuracy.

\paragraph{Importance weighted sampling for stochastic gradient methods.}
At a high level, this work adopts the algorithmic idea of importance-weighted sampling for stochastic gradient methods.~\cite{gopal2016adaptive} proposes a new mechanism for sampling training instances for stochastic gradient
descent methods. Specifically,  the sampling weights are proportional to the L2 norm of the gradient. They claim this is the way to minimize the total variance of the descent direction.~\cite{zhao2015stochastic} also uses importance sampling with weight proportional to the norm of the stochastic gradient, to minimize the variance of the stochastic gradient.~\cite{needell2014stochastic} shows that in SGD for smooth and strongly convex objectives, re-weighting the sampling distribution improves the rate of convergence, where it proposes to use the weight proportional to the smoothness parameter.

\paragraph{Negative results on efficient learning halfspace with noise.}
Toward designing efficient algorithms in learning halfspaces, there have been several notable trials in understanding the possibility via convex surrogate loss minimization. Unfortunately, it has been found and discussed that such a natural and intuitive approach is unable to 
achieve the PAC learning guarantee. 

Specifically,~\cite{diakonikolas2019distribution} constructs discrete distribution supported on two points and argues that, for any decreasing convex loss function, the minimizer of expected loss 
has a misclassification error lower bound, even with a margin assumption.
The arguments in~\cite{awasthi2015efficient} and~\cite{awasthi2016learning} assume uniform distribution on the unit $\RR^2$ ball, where~\cite{awasthi2015efficient} shows that 
the excess error of the hinge loss minimizer will not get arbitrarily small even with unlimited sample complexity 
under Massart noise condition, and further,~\cite{awasthi2016learning} proves that convex surrogate loss minimization does not work under Massart noise condition, for a family of surrogate losses, including most commonly used loss functions. 

Although hardness results have been discovered for convex surrogate loss minimization to learn halfspaces with excess error arbitrarily close to the error of Bayes optimal halfspace $\opt$, such approaches can achieve ``approximate'' learning halfspaces. Specifically,~\cite{frei2021agnostic} shows that under log-concave distribution, convex surrogate loss minimization achieves a population risk of $\tilde{O} (\opt^{\frac12})$, and~\cite{ji2022agnostic} provides a matching lower bound by constructing a well-behaved distribution where the minimizer of logistic loss achieves a misclassification error of $\Omega (\opt^{\frac12})$. 

Besides the above algorithm-specific hardness results, algorithm-independent results for learning halfspaces with noise have also been discovered.  
While in the realizable setting, the hypothesis class of halfspaces is efficiently learnable~\citep{maass1994fast}, with the presence of noise, the learning problem is tremendously more challenging. In the agnostic model, where the error of the Bayes classifier is known, and the corruption can be adversarial, learning halfspaces is known to be computationally hard~\citep{guruswami2009hardness,daniely2016complexity}. Recent results~\citep{diakonikolas2020near,goel2020statistical} establish the computational hardness in agnostically learning halfspaces by showing Statistical Query lower bounds of $d^{\poly(1/\epsilon)}$, even under the Gaussian distribution. With Massart noise,~\cite{diakonikolas2022near} shows a distribution free lower bound that no efficient Statistical Query algorithm can achieve an error better than $\Omega(\eta)$. This result is later strengthened in~\cite{diakonikolas2022cryptographic} where it proves that assuming the subexponential time hardness of the Learning with Errors (LWE) problem, no efficient algorithm can achieve an error better than $\Omega(\eta)$ in the same setting. 
These hardness results motivate us to define and study benign noise and unlabeled distribution conditions for efficient learning halfspace.

\section{Additional Notations}
Throughout the Appendix Section, we use $\tilde{O}(\cdot), \tilde{\Theta}(\cdot)$ to hide factors of the form $\polylog(d, \frac{1}{\epsilon})$ and $\poly(R, U, L)$.

\section{Key lemmas}

\begin{lemma}[Restatement of Lemma~\ref{lem:active-SGD}]
Let the expected loss function $L_\sigma(w) = \EE \phi_\sigma \rbr{y \frac{\inner{w}{x}}{\|w\|}}$. 
If Algorithm~\ref{alg:active-PSGD-finding-stationary-point} receives inputs
$N = \tilde{O}(\frac{d}{\sigma^2 \rho^4})$, $\beta = \tilde{\Theta}(\frac{ \rho^2 \sigma^2}{d})$,
then its output $w_R$ is a unit vector and satisfies that, with probability at least $\frac12$, 
\[
    \|\nabla L_\sigma(w_R) \| 
    \leq
    \rho
    \]
Furthermore, during $N$ iterations, with probability at least $1-\frac{\delta}{6S}$, the total number of label queries is at most $\tilde{O}(\frac{d}{\sigma \rho^4} + \sqrt{\frac{d}{\sigma^2 \rho^4} \ln \frac{6S}{\delta}} )$. 

\end{lemma}
\begin{proof}
    Let $L = \tilde{O}(\frac{1}{\sigma} )$ be such that $L_\sigma$ is $L$-smooth; let $B^2 = \tilde{O} (\frac{d}{\sigma})$ 
    be such that $\EE_{(x,y) \sim D} \sbr{\|g_w \|^2} \leq B^2$; the existence of $L$ and $B$ is guaranteed by Lemma~\ref{lem:smoothness-factor-ub} and item~\ref{item:variance} of Lemma~\ref{lem:active-oracle-property}.

    Define a filtration $\cbr{\Fcal_i}_{i=0}^N$, where $\Fcal_i$ denotes the $\sigma$-field $\sigma(g_1, g_2, \ldots, g_i)$. We use $\EE_i [\cdot]$ to denote the conditional expectation with respect to $\Fcal_i$. 

    Denote by $\Wcal = \cbr{w \in \RR^d: \|w\| \geq 1 }$.
    Note that $v_i - w_{i-1} =  - \beta g_i$. Line~\ref{line:normalization} of Algorithm~\ref{alg:active-PSGD-finding-stationary-point} ensures that $\|w_i\| = 1$, for all $i = 1, \ldots, N$. Further, by item~\ref{item:perp} of Lemma~\ref{lem:active-oracle-property}, $g_i$ is perpendicular to $w_{i-1}$, hence $\|v_i\|^2 = \|w_{i-1}\|^2 + \|\beta g_i \|^2 \geq 1$, that is, $v_i \in \Wcal$. 
    For any $t \in [0,1]$, $\| w_{i-1} + t (v_i - w_{i-1}) \|^2 = \|w_{i-1}\|^2 + \|t \beta g_i \|^2 \geq 1$. 
    Therefore, the line segment between $v_i$ and $w_{i-1}$ lies in $\Wcal$.  
    
    Hence we have for $i=1,2,\ldots,N$, 
    \begin{align*}
        L_\sigma(v_i) - L_\sigma(w_{i-1})
        =&
        \int_{0}^1 \inner{\nabla L_\sigma \rbr{ w_{i-1} + t(v_i - w_{i-1})}}{(v_i - w_{i-1})} \diff t \\
        =&
        \inner{\nabla L_\sigma \rbr{ w_{i-1}}}{(v_i - w_{i-1})}
        + \int_{0}^1 \inner{\nabla L_\sigma \rbr{ w_{i-1} + t(v_i - w_{i-1})} - \nabla L_\sigma \rbr{ w_{i-1}}}{(v_i - w_{i-1})} \diff t \\
        \leq&
        - \beta \inner{\nabla L_\sigma \rbr{ w_{i-1}}}{g_i} + \int_{0}^1 Lt \|v_i - w_{i-1}\|^2 \diff t \\
        =&
        - \beta \inner{\nabla L_\sigma \rbr{ w_{i-1}}}{g_i} + \frac{\beta^2 L}{2} \|g_i\|^2
    \end{align*}
    where the first equality is by Newton-Leibniz formula, the inequality is by Cauchy-Schwarz inequality
    and the following reasoning: by multivariable mean value theorem, 
    $\nabla L_\sigma \rbr{ w_{i-1} + t(v_i - w_{i-1})} - \nabla L_\sigma \rbr{ w_{i-1}} = M t (v_i - w_{i-1})$, where $M$ is
    the Hessian matrix of $L_\sigma$ evaluated at some point on the line segment between $v_i$ and $w_{i-1}$. 
    Since the line segment between $v_i$ and $w_{i-1}$ lies in $\Wcal$, 
    together with Lemma~\ref{lem:smoothness-factor-ub}, we have 
    $\| \nabla L_\sigma \rbr{ w_{i-1} + t(v_i - w_{i-1})} - \nabla L_\sigma \rbr{ w_{i-1}} \|  = \|M t (v_i - w_{i-1})\|
    \leq \|M\|_{\mathrm{op}} \| t (v_i - w_{i-1})\| \leq L t \|v_i - w_{i-1}\|$, for all $i \in [N]$.

    Since
    $\phi_\sigma$ is invariant under positive scaling: for any $w \neq 0$, $\alpha > 0$, $\phi_\sigma(\alpha w) = \phi_\sigma(w)$,
    we have 
    $L_\sigma$ is invariant under positive scaling as well. Hence 
    $L_\sigma(w_i) = L_\sigma(\frac{v_i}{\|v_i\|_2}) = L_\sigma(v_i)$. Therefore, for all $i=1,2,\ldots,N$,  
    \begin{equation}
    L_\sigma(w_i) - L_\sigma(w_{i-1})
    \leq
    - \beta \inner{\nabla L_\sigma \rbr{ w_{i-1}}}{g_i} + \frac{\beta^2 L}{2} \|g_i\|^2      
    \label{eqn:one-iteration}
    \end{equation}

    Summing up the above inequalities through $i = 1, \ldots, N$, we have 
    \begin{equation}
    \sum_{i=1}^N \beta \inner{\nabla L_\sigma \rbr{ w_{i-1}}}{g_i} 
    \leq
    L_\sigma(w_0) - L_\sigma(w_N) + \frac{\beta^2 L}{2} \sum_{i=1}^N \|g_i\|^2  
    \leq
    1 + \frac{\beta^2 L}{2} \sum_{i=1}^N \|g_i\|^2 
    \label{eqn:summed-up}
    \end{equation}
    where the last inequality follows from $0 \leq L_\sigma(w) \leq 1, \forall w \in \RR^d$, and we have $L_\sigma(w_0) - L_\sigma(w_n) \leq 1$. 

    Taking expectation on both sides, and by linearity of expectation, we have:
    \[
    \beta \sum_{i=1}^N \EE\sbr{  \inner{\nabla L_\sigma \rbr{ w_{i-1}}}{g_i} }
    \leq 
    1 + \frac{\beta^2 L}{2} \sum_{i=1}^N \EE\sbr{\|g_i\|^2}. 
    \]

    For the left hand side, applying item~\ref{item:unbiased} of Lemma~\ref{lem:active-oracle-property} and the law of iterated expectation, we have that for all $i = 1, \ldots, N$, 
    \begin{equation*}
    \EE\sbr{  \inner{\nabla L_\sigma \rbr{ w_{i-1}}}{g_i} }
    =
    \EE\sbr{ \EE_{i-1} \inner{\nabla L_\sigma \rbr{ w_{i-1}}}{g_i} }
    =
    \EE\sbr{ \|\nabla L_\sigma \rbr{ w_{i-1}}\|^2 }
    \end{equation*}

    For the right hand side, applying item~\ref{item:variance} of Lemma~\ref{lem:active-oracle-property} and the law of iterated expectation, we have  that for all $i = 1, \ldots, N$, 
    \[
    \EE\sbr{ \|g_i\|^2 } = \EE\sbr{ \EE_{i-1} \|g_i\|^2 } \leq B^2
    \]
    
    Therefore, we have 
    \[
    \beta \EE\sbr{ \sum_{i=1}^N  \|\nabla L_\sigma \rbr{ w_{i-1}}\|^2}
    \leq
    1 + \frac{\beta^2 L}{2} N  B^2
    \]

    Note that we are choosing $R$ to be uniformly distributed on $\cbr{0, \ldots, N-1}$, hence by the law of iterated expectation,
    \[
    \EE \sbr{\|\nabla L_\sigma (w_R) \|^2 }
    =
    \EE \sbr{ \EE \sbr{\|\nabla L_\sigma (w_R) \|^2 \mid w_1, \ldots, w_{N-1} } }
    =
    \EE\sbr{ \frac{\sum_{i=1}^N \|\nabla L_\sigma \rbr{ w_{i-1}}\|^2}{N } }
    \leq
    \frac{1}{\beta N } \sbr{1 + \frac{\beta^2 L}{2} NB^2  }
    \]

    By the definitions of $\beta,N, B$ and $L$, 
    the above inequality gives us
    \[
    \EE \sbr{\|\nabla L_\sigma(w_R) \|^2 }
    \leq
    \frac{\rho^2}{2}
    \]

    By Markov's inequality, we have \[
    \PP \sbr{\|\nabla F(w_R) \|^2 \geq \rho^2 }
    \leq \frac12
    \]
    
    That is, with probability at least $\frac12$, \[
    \|\nabla F(w_R) \| 
    \leq
    \rho
    \]

    Finally, by Lemma~\ref{lem:high-prob-label}, with probability at least $1-\frac{\delta}{6S}$, the total number of label queries after $N$ calls to \afo is at most $\tilde{O}(\sigma N + \sqrt{N \ln \frac{6S}{\delta}} ) = \tilde{O}(\frac{d}{\sigma \rho^4} + \sqrt{\frac{d}{\sigma^2 \rho^4} \ln \frac{6S}{\delta}} )$.
\end{proof}

\begin{lemma}[Restatement of Lemma~\ref{lem:selection-1}]
     Suppose $w_1, \ldots, w_S$ are such that $\min_i \|\nabla L_\sigma (w_i) \| \leq \rho$, then after executing lines~\ref{line:selection-1-start} to~\ref{line:selection-1-end} of Algorithm~\ref{alg:main}, with \[
    M_1 =
    c \frac{d}{\sigma^2 \rho^2} \ln\frac{S}{\delta}
    \]
    for some constant $c$,
    with probability at least $1-\delta/6$, $\tilde{w}$ satisfies \[
     \|\nabla L_\sigma (\tilde{w}) \| \leq 2\rho
     \]
     Furthermore, after $M_1$ calls to \afo, with probability at least $1-\frac{\delta}{6S}$, the total number of label queries is at most $\tilde{O}(\frac{d}{\sigma \rho^2} \ln\frac{S}{\delta} + \sqrt{\frac{d}{\sigma^2 \rho^2} \ln\frac{S}{\delta} \ln \frac{6S}{\delta}} )$. 
\end{lemma}

The algorithmic idea  underlying lines~\ref{line:selection-1-start} to~\ref{line:selection-1-end} of Algorithm~\ref{alg:main} for iterate selection is largely inspired by Corollary 2.5 in~\cite{ghadimi2013stochastic}. However, here we rely on the sub-exponential-ness of the stochastic gradient outputted by \afo (See Lemma~\ref{lem:sub-exponential}), whereas Corollary 2.5 in~\cite{ghadimi2013stochastic} assumes sub-gaussian stochastic gradient. 
We include the proof for completeness. 

\begin{proof}
    Recall from Algorithm~\ref{alg:main} that $\bar{g}_s = \frac{1}{M_1} \sum_{i=1}^{M_1} \|g_{s,i}\|$ and $s^* = \argmin_{s \in [S]} \bar{g}_s$. We have
    \begin{align}
        \|\bar{g}_{s^*}\|
        =&
        \min_s \|\bar{g}_s\|  \nonumber \\
        =&
        \min_s \|\nabla L_\sigma(w_s) + \bar{g}_s - \nabla L_\sigma(w_s)   \| \nonumber\\
        \leq&
        \min_s \rbr{  \|\nabla L_\sigma(w_s) \| + \|\bar{g}_s - \nabla L_\sigma(w_s)   \|  } \nonumber\\
        \leq&
        \min_s \|\nabla L_\sigma(w_s) \| + \max_s  \|\bar{g}_s
        - \nabla L_\sigma(w_s)   \|  
        \label{eqn:empirical-gradient-ub}
    \end{align}
    where the first inequality is by the triangle inequality, the second inequality is by the fact that $\min_i (a_i + b_i) \leq \min_i a_i + \max_i b_i$, Further, 
    \begin{align}
        \|\nabla L_\sigma (\tilde{w}) \| 
        =&
        \| \bar{g}_{s^*} + \nabla L_\sigma (\tilde{w}) - \bar{g}_{s^*} \| \nonumber\\
        \leq&
        \| \bar{g}_{s^*} \| + \| \nabla L_\sigma (\tilde{w}) - \bar{g}_{s^*} \| 
        \label{eqn:true-gradient-ub}
    \end{align}
    where the inequality is by the triangle inequality. 

    Denote $\delta_{s, i} := g_{s,i} - \nabla L_\sigma(w_s)$, for $s = 1, \ldots, S, i = 1, \ldots, M_1$. We have $\bar{g}_s - \nabla L_\sigma(w_s)  =  \frac{1}{M_1} \sum_{i=1}^{M_1} g_{s,i} - \nabla L_\sigma(w_s) = \frac{1}{M_1} \sum_{i=1}^{M_1} \delta_{s, i}$, for all $s = 1, \ldots, S$. 

    By Lemma~\ref{lem:sub-exponential}, for any unit vector $w$, 
    $\|g_w\|$ is sub-exponential with parameter $K = \tilde{\Theta}(\frac{\sqrt{d}}{\sigma})$. Together with Proposition 2.7.1 in~\cite{vershynin2018high} on equivalent characterizations of sub-exponential random variables, we have 
    $\EE \sbr{ \exp(\frac{\|g_w \|}{K}) \mid w } \leq \exp(1)$
    for any unit vector $w$. 
    
    Applying Lemma~\ref{lem:sub-exp-concentration}, we have for any $s = 1, \ldots, S$ and $\lambda >0$, 
    \[
    \PP \cbr{\|\sum_{i=1}^{M_1} \delta_{s, i} \| \geq \sqrt{2}(\sqrt{e} + \lambda ) \sqrt{M_1} K  }
    \leq
    2 \exp \cbr{-\frac{1}{64} \min \sbr{ \lambda^2, 16 \sqrt{M_1} \lambda} }
    \]

    Taking $\lambda_0 = \max\cbr{8 \sqrt{\ln \frac{12S}{\delta}}, \frac{4}{\sqrt{M_1}} \ln \frac{12S}{\delta} }$, we have for any $s = 1, \ldots, S$ \[
    \PP \cbr{\|\bar{g}_s - \nabla L_\sigma(w_s)\| \geq \sqrt{2}(\sqrt{e} + \lambda_0 ) \frac{1}{\sqrt{M_1}} K  }
    =
    \PP \cbr{\|\sum_{i=1}^{M_1} \delta_{s, i} \| \geq \sqrt{2}(\sqrt{e} + \lambda_0 ) \sqrt{M_1} K  }
    \leq
    \frac{\delta}{6S}
    \]

    Taking a union bound over $s = 1, \ldots, S$, we have with probability at least $1-\delta/6$, \[
    \max_{s = 1, \ldots, S} \|\bar{g}_s - \nabla L_\sigma(w_s)\| 
    \leq 
    \sqrt{2}(\sqrt{e} + \lambda_0 ) \frac{1}{\sqrt{M_1}} K 
    \]

    In conjunction with Equations~\eqref{eqn:empirical-gradient-ub} and~\eqref{eqn:true-gradient-ub}, and recall that $M_1 = c \frac{d}{\sigma^2 \rho^2} \ln\frac{S}{\delta}$ for some large enough constant $c$, we have 
    with probability at least $1 - \delta/6$, $\tilde{w}$ satisfies
    \begin{align*}
    \|\nabla L_\sigma (\tilde{w}) \| 
    \leq& 
    \| g(\tilde{w}) \| + \| \nabla L_\sigma (\tilde{w}) - g(\tilde{w}) \| \\ 
    \leq&
    \min_s \|\nabla L_\sigma(w_s) \| + \max_s  \|g(w_s) - \nabla L_\sigma(w_s)   \| + \| \nabla L_\sigma (\tilde{w}) - g(\tilde{w}) \|         \\
    \leq&
    \rho + 2\sqrt{2}(\sqrt{e} + \lambda_0 ) \frac{1}{\sqrt{M_1}} K \\ 
    \leq&
    \rho + 2\sqrt{2}(\sqrt{e} + 8 \sqrt{\ln \frac{6S}{\delta}} + \frac{4}{\sqrt{M_1}} \ln \frac{6S}{\delta}  ) \frac{1}{\sqrt{M_1}} K \\ 
    \leq&
    2 \rho
    \end{align*}
    Lastly, 
    by Lemma~\ref{lem:high-prob-label}, with probability at least $1-\frac{\delta}{6S}$, the total number of label queries after $M_1$ calls to \afo is at most $\tilde{O}(\sigma M_1 + \sqrt{M_1 \ln \frac{6S}{\delta}} ) = \tilde{O}(\frac{d}{\sigma \rho^2} \ln\frac{S}{\delta} + \sqrt{\frac{d}{\sigma^2 \rho^2} \ln\frac{S}{\delta} \ln \frac{6S}{\delta}} )$.
    \end{proof}

\begin{lemma}[Restatement of Lemma~\ref{lem:selection-2}]
    Suppose $\tilde{w}$ satisfies that $\exists w \in \cbr{\pm \tilde{w} }$, such that $\err(w) - \err(w^*) \leq \epsilon$ with $\epsilon \leq \frac12 \alpha (\frac{1}{A})^\frac{1-\alpha}{\alpha}$, then after executing lines~\ref{line:selection-2-start} to~\ref{line:selection-2-end} of Algorithm~\ref{alg:main}, where $M_2 =  O\rbr{(\frac{2}{\alpha (\frac{1}{A})^\frac{1-\alpha}{\alpha}} )^2 \ln \frac{6}{\delta} } $, we have
    with probability at least $1-\delta/3$, $\hat{w}$ satisfies \[
    \err(\hat{w}) - \err(w^*) \leq \epsilon
    \]
\end{lemma}
\begin{proof}
    From Lemma~\ref{lem:bayes-error}, we know that $\err(w^*) \leq \frac12 -\alpha (\frac{1}{A})^\frac{1-\alpha}{\alpha}$, hence we have $\exists \bar{w} \in \cbr{\pm \tilde{w} }$, such that \[
    \err(\bar{w}) \leq  \err(w^*) +  \epsilon \leq \frac12 -\frac12\alpha (\frac{1}{A})^\frac{1-\alpha}{\alpha}
    \]

    For any $(x,y)$ and any $w \in \RR^d$, exactly one of $\cbr{\pm w}$ will label $(x,y)$ correctly. Thus for any $w \in \RR^d$, $\err(w) + \err(-w) = 1$, and $\err_S(w) + \err_S(-w) = 1$. So we have 
    \[
    \err(-\bar{w}) \geq \frac12 + \frac12\alpha (\frac{1}{A})^\frac{1-\alpha}{\alpha}
    \]

    By Hoeffding's inequality, drawing $M_2 =  O\rbr{(\frac{2}{\alpha (\frac{1}{A})^\frac{1-\alpha}{\alpha}} )^2 \ln \frac{6}{\delta} } $ iid labeled examples from $D$ as the validation set $S$, we have that 
    with probability at least $1-\delta/3, \abs{\err(w) - \err_S(w)} \leq \frac12\alpha (\frac{1}{A})^\frac{1-\alpha}{\alpha}, \forall w \in \cbr{\pm \tilde{w} }$. 
    
    This means that with probability at least $1-\delta/3$, $\err_S(\bar{w}) < 1/2$, which implies that 
    $\hat{w} = \bar{w}$, hence $\err(\hat{w}) - \err(w^*) \leq \epsilon$. 
\end{proof}

\begin{lemma}[Restatement of Lemma~\ref{lem:gradient-logistic-TNC}]
Let $D_X$ be a well-behaved distribution, and $D$ satisfies $(A, \alpha)$-TNC.  
Recall that $L_\sigma(w) = \EE_D \sbr{\phi_\sigma \rbr{y \frac{\inner{w}{x} }{\|w\|_2}}}$ where $\phi_\sigma$ is softmax loss.
Let $w$ be a unit vector
such that $\theta(w, w^*) \in (\theta, \pi - \theta)$, where $\theta \leq (\frac{1}{4})^{\frac{3\alpha-1}{2(1-\alpha)}} (\frac{128U}{cR^2L})^{\frac12} = \Theta(A)$.
Then for $\sigma =  \Theta \rbr{(\frac{1}{A})^\frac{1-\alpha}{3\alpha-1}\theta^\frac{2\alpha}{3\alpha-1}}$, we have that 
$\| \nabla_w L_\sigma(w) \|_2
\geq 
\Omega \rbr{
(\frac{1}{A})^\frac{2(1-\alpha)}{3\alpha-1} \theta^\frac{2(1-\alpha)}{3\alpha-1} }$. 
\end{lemma}

\begin{proof}
With foresight, we choose $\sigma = \rbr{\frac{1}{768 U} \cdot R^2 L (\frac{R L}{A})^\frac{1-\alpha}{\alpha} }^\frac{\alpha}{3\alpha-1} \theta^\frac{2\alpha}{3\alpha-1} = \Theta \rbr{(\frac1A)^\frac{1-\alpha}{3\alpha-1}\theta^\frac{2\alpha}{3\alpha-1}}$. 
By our assumption that $\theta \leq \frac{8A}{RL} \cdot (\frac{1}{4})^{\frac{3\alpha-1}{2(1-\alpha)}} (\frac{768 U}{R^2L})^{\frac12} = \Theta(A)$, 
\begin{equation}
\sigma \leq \frac{8A}{RL} (\frac{1}{4})^{\frac{\alpha}{1-\alpha}}.
\label{eqn:sigma-ub}
\end{equation}

Without loss of generality, suppose 
$w = (0,1,0,\ldots,0)$ and $w^* = (-\sin \theta, \cos \theta,0,\ldots,0)$, 
we have 
\begin{align}
   \| \nabla_w L_\sigma(w) \|_2
    =&
    \normx{ \nabla \EE_D \sbr{\phi_\sigma \rbr{y \frac{\inner{w}{x} }{\|w\|_2}} } }_2  \nonumber\\
    =&
    \normx {\EE_D \sbr{\phi_\sigma' \rbr{y \frac{\inner{w}{x} }{\|w\|_2}} y \rbr{\frac{x}{\|w\|_2} - \frac{\inner{w}{x} w}{\|w\|_2^3}} } }_2 \nonumber\\
    =&    
    \normx{ \EE_x \sbr{ \phi'_\sigma(x_2) (1-2 \eta(x)) \sign(\inner{w^*}{x}) \rbr{\frac{x}{\|w\|_2} - \frac{\inner{w}{x} w}{\|w\|_2^3}}} }_2 \nonumber\\
    \geq& 
    \abs{\EE_x \sbr{ \phi'_\sigma(x_2) (1-2 \eta(x)) \sign(\inner{w^*}{x}) x_1}} 
    \label{eqn:gradient-norm}
\end{align}
where the first equality is by taking the gradient of $L_\sigma(w) = \EE_D \sbr{\phi_\sigma \rbr{y \frac{\inner{w}{x} }{\|w\|_2}}}$, the second equality is by $\nabla \phi_\sigma \rbr{y \frac{\inner{w}{x} }{\|w\|_2}} = \phi_\sigma' \rbr{y \frac{\inner{w}{x} }{\|w\|_2}} y \rbr{\frac{x}{\|w\|_2} - \frac{\inner{w}{x} w}{\|w\|_2^3}}$, 
the third inequality is by noting that $\frac{\inner{w}{x} }{\|w\|_2} = x_2 $ and $\phi'_\sigma(t) = \phi'_\sigma(-t)$, and $\EE \sbr{y \mid x} = (1-2 \eta(x)) \sign(\inner{w^*}{x})$, and the last inequality is because $\frac{x}{\|w\|_2} - \frac{\inner{w}{x} w}{\|w\|_2^3}  = (x_1, 0, x_3, \ldots)$.

Denote by $G := \cbr{x \in \RR^d : \phi'_\sigma(x_2) (1-2 \eta(x)) \sign(\inner{w^*}{x}) x_1 \geq 0} = 
\cbr{x \in \RR^d : \sign(\inner{w^*}{x}) x_1 \leq 0}$, and $G^C = \RR^d \setminus G$, then we have, 

\begin{align}
    \| \nabla_w L_\sigma(w) \|_2
    \geq&
    \abs{\EE_x \sbr{ \phi'_\sigma(x_2) (1-2 \eta(x)) \sign(\inner{w^*}{x}) x_1}} \nonumber \\
    =&
    \abs{\EE_x \sbr{ \phi'_\sigma(x_2) (1-2 \eta(x)) \sign(\inner{w^*}{x}) x_1 \rbr{\ind(x \in G) + \ind(x \in G^C) }}} \nonumber\\
    \geq&
    \EE_x \sbr{ \abs{\phi'_\sigma(x_2) } (1-2 \eta(x)) |x_1| \ind(x \in G)}
    -
    \EE_x \sbr{\abs{\phi'_\sigma(x_2) } (1-2 \eta(x)) |x_1| \ind(x \in G^C)} \nonumber \\
    =&
    \EE_x \sbr{ \abs{\phi'_\sigma(x_2) } (1-2 \eta(x)) |x_1|}
    -
    2 \EE_x \sbr{\abs{\phi'_\sigma(x_2) } (1-2 \eta(x)) |x_1| \ind(x \in G^C)} 
    \label{eqn:gradient-lb}
\end{align}
where the first inequality is from Equation~\eqref{eqn:gradient-norm}, the equalities are because $\ind(x \in G) + \ind(x \in G^C) = 1$, for all $x \in \RR^d$, 
the last inequality is by the triangle inequality. 

We lower bound $\EE_x \sbr{ \abs{\phi'_\sigma(x_2) } (1-2 \eta(x)) |x_1|}$ as follows. 

Define $R_1 = \cbr{x \in \RR^d: x_1 \in [\frac{R}{4}, \frac{R}{2}], x_2 \in [0, \sigma]}$, 
we lower bound $\PP_x(x \in R_1)$ as follows. 
We project $x$ onto the 2-dimensional subspace $V$
spanned by $\cbr{e_1, e_2}$; define $\tilde{x}$ to be the coordinate of its projection, and let $\tilde{R}_1$ be the projection of $R_1$ onto $V$. 
Denote by $\tilde{D}_X$ the distribution of $\tilde{x}$, and denote by its probability density function $p_V$.
Since $D_X$ is well-behaved, we have \[
\PP_x(x \in R_1) 
= \PP_{\tilde{x} \sim \tilde{D}_X }(\tilde{x} \in  \tilde{R}_1)
= \int_{\tilde{R}_1} p_V(\tilde{x}) d\tilde{x}
\geq \frac{R}{4} \sigma L 
\]

Let $t = 2 (\frac{R \sigma L}{8 A})^\frac{1-\alpha}{\alpha}$; with this choice of $t$, 
$\frac{R}{8} \sigma L \geq A (\frac{t}{2})^\frac{\alpha}{1-\alpha}$. Also note that by Eq.~\eqref{eqn:sigma-ub}, $t \leq \frac12$. We obtain the following, 
\begin{align}
    \EE_x \sbr{\abs{\phi'_\sigma(x_2) } (1-2 \eta(x)) |x_1|}
    \geq&
    \EE_x \sbr{\abs{\phi'_\sigma(x_2) } (1-2 \eta(x)) |x_1| \ind(x \in R_1)} \nonumber\\
    \geq&
    \frac{1}{6\sigma} \cdot t \cdot \EE \sbr{\ind(1-2\eta(x) \geq t) |x_1| \ind(x \in R_1)} \nonumber\\
    \geq&
    \frac{1}{6\sigma} \cdot t \frac{R}{4} \cdot \EE \sbr{\ind(1-2\eta(x) \geq t) \ind(x \in R_1)} \nonumber\\
    \geq&
    \frac{1}{6\sigma} \cdot t \frac{R}{4} \cdot \sbr{\PP(x \in R_1)- \PP(1-2\eta(x) \leq t)} \nonumber\\
    \geq&
    \frac{1}{6\sigma} \cdot t \frac{R}{4} \cdot \sbr{\PP(x \in R_1)- A (\frac{t}{2})^\frac{\alpha}{1-\alpha}} \nonumber\\
    \geq&
    \frac{1}{6\sigma} \cdot t \frac{R}{4} \cdot \sbr{\frac{R}{4} \sigma L - \frac{R}{8} \sigma L} \nonumber\\
    =& 
    \frac{1}{192} c \cdot R^2 L t \nonumber\\
    =& 
    \frac{1}{192} c \cdot R^2 L \cdot
    2 (\frac{R \sigma L}{8 A})^\frac{1-\alpha}{\alpha}
    \label{eqn:gradient-1}
\end{align}
where the first inequality is since $R_1 \subset \RR^d$, 
the second inequality is
by noting that when $|x_2| \leq \sigma$, $\abs{\phi'_\sigma(x_2) } \geq \frac1\sigma \frac{e}{(1+e)^2} \geq \frac{1}{6\sigma}$.
The third is because for all $x \in R_1$, $|x_1| \geq \frac{R}{4}$, the fourth is by basic logic operation, the fifth is by the definition of TNC and $t \leq \frac12$.  
The sixth is by $\PP_x(x \in R_1) \geq \frac{R}{4} \sigma L $ and $\frac{R}{8} \sigma L \geq A (\frac{t}{2})^\frac{\alpha}{1-\alpha}$. 
The other inequalities and equalities are all by algebra. 

Next, we upper bound $\EE_x \sbr{\abs{\phi'_\sigma(x_2) }  (1-2 \eta(x)) |x_1| \ind(x \in G^C)}$. 

Let $f(r \cos \varphi, r \sin \varphi)$ denote the density function after projection on the 2-d subspace spanned by $\cbr{e_1, e_2}$, 
\begin{align}
    \EE_x \sbr{\abs{\phi'_\sigma(x_2) }  (1-2 \eta(x)) |x_1| \ind(x \in G^C)}
    \leq&
    \EE_x \sbr{\abs{\phi'_\sigma(x_2) }  |x_1| \ind(x \in G^C)} \nonumber\\
    =&
    \EE_x \sbr{ \frac{1}{\sigma} \frac{e^{\frac{|x_2|}{\sigma}}}{(1+ e^{\frac{|x_2|}{\sigma}})^2} |x_1| \ind(x \in G^C)} \nonumber\\
    =&
    \EE_x \sbr{ \frac{1}{\sigma} \frac{e^{-\frac{|x_2|}{\sigma}}}{(1+ e^{-\frac{|x_2|}{\sigma}})^2} |x_1| \ind(x \in G^C)} \nonumber\\
    \leq&
    \EE_x \sbr{ \frac{1}{\sigma} e^{-\frac{|x_2|}{\sigma}} |x_1| \ind(x \in G^C)} \nonumber\\
    =&
    \frac{2}{\sigma}
    \int_0^\infty \int_\theta^{\frac{\pi}{2}} f(r \cos \varphi, r \sin \varphi) r^2 \cos \varphi e^{-\frac{r \sin \varphi}{\sigma}} \dif \varphi \dif r \nonumber\\
    \leq&
    \frac{2}{\sigma} U
    \int_0^\infty \int_\theta^{\frac{\pi}{2}} r^2 \cos \varphi e^{-\frac{r \sin \varphi}{\sigma}} \dif \varphi \dif r \nonumber\\
    =& 
    2 U \frac{\sigma^2}{\tan^2 \theta} \nonumber\\
    \leq&
    2 U \frac{\sigma^2}{ \theta^2} 
    \label{eqn:gradient-2}
\end{align}
where the first inequality is by $\eta(x) \geq 0$, 
the second equality uses $\phi_\sigma'(t) = -\frac{1}{\sigma}\frac{e^{\frac{t}{\sigma}}}{(1 +e^{\frac{t}{\sigma}})^2}$, the third equality is by algebra, the fourth inequality is because $(1+ e^{-\frac{|x_2|}{\sigma}})^2 \geq 1$, 
the fifth equality is writing the expectation as the integral in the polar coordinate, the sixth inequality is by the definition of well-behaved distribution: $f(r \cos \varphi, r \sin \varphi) \leq U$, for all $r \geq 0, \varphi \in [0, 2\pi]$, the next equality is by algebra, the last inequality is by the elementary fact that $\tan \theta \geq \theta$, for all $\theta \in [0, \pi/2]$. 

Therefore, putting together Equations \eqref{eqn:gradient-lb}, \eqref{eqn:gradient-1} and \eqref{eqn:gradient-2}, we obtain  \[
\| \nabla_w L_\sigma(w) \|_2
\geq 
\frac{1}{192} \cdot R^2 L \cdot
    2 (\frac{R \sigma L}{8 A})^\frac{1-\alpha}{\alpha} 
- 4 U \frac{\sigma^2}{ \theta^2} 
\]

Recall the choice that $\sigma = \rbr{\frac{1}{768 U} \cdot R^2 L (\frac{R L}{A})^\frac{1-\alpha}{\alpha} }^\frac{\alpha}{3\alpha-1} \theta^\frac{2\alpha}{3\alpha-1} = \Theta \rbr{(\frac1A)^\frac{1-\alpha}{3\alpha-1}\theta^\frac{2\alpha}{3\alpha-1}}$, 
then we obtain

\[
\| \nabla_w L_\sigma(w) \|_2
\geq 
4 U 
\rbr{\frac{1}{768 U} \cdot R^2 L (\frac{R L}{A})^\frac{1-\alpha}{\alpha} }^\frac{2\alpha}{3\alpha-1} \theta^\frac{2-2\alpha}{3\alpha-1}
=
\Omega \rbr{
(\frac1A)^\frac{2(1-\alpha)}{3\alpha-1} \theta^\frac{2(1-\alpha)}{3\alpha-1} }
\]

\end{proof}

\begin{lemma}[Restatement of Lemma~\ref{lem:active-oracle-property}]
    Let $g_w$ be the random output of $\afo(w)$. We have, for any unit vector $w$:
    \begin{enumerate}
        \item $g_w$ is perpendicular to $w$;
        \item $g_w$ is an unbiased estimator of $\nabla L_\sigma(w): \EE \sbr{g_w}= \nabla L_\sigma(w)$;
        \item $\EE \sbr{\|g_w \|^2} \leq \tilde{O} (\frac{d}{\sigma})$;
        \item The expected number of label queries per call to \afo is $\tilde{O}(\sigma)$.
    \end{enumerate}
\label{lem:active-oracle-property-restate}

\end{lemma}

\begin{proof}
    Recall that $q(w,x) = \sigma \abs{\phi_\sigma'(\inner{\frac{w}{\|w\|}}{x} ) }$, and $h(w, x, y) =
        - \frac{1}{\sigma} y \rbr{\frac{x}{\|w\|_2} - \frac{\inner{w}{x} w}{\|w\|_2^3} } $.

    We prove the first term as follows. 
    Note that $g_w$ can take the value of $0$ or $h(w, x, y)$. If $g_w = 0$, then obviously, $\inner{g_w}{w} = 0$. If $g_w = h(w, x, y)$, we have
    \[
    \inner{h(w, x, y)}{w} =
        - \frac{1}{\sigma} \inner{\frac{x}{\|w\|_2} - \frac{\inner{w}{x} w}{\|w\|_2^3} }{w} = 
        - \frac{1}{\sigma} \rbr{\frac{\inner{x}{w}}{\|w\|_2} - \frac{\inner{x}{w} \|w\|^2}{\|w\|_2^3}  } = 0 .
        \]
    Hence we conclude that in both cases, $g_w$ is perpendicular to $w$.
    
    For the second item, let $w \in \RR^d$, we have 
    \begin{align*}
    \EE \sbr{g_w}
    =&
    \EE_{(x,y) \sim D, Z\sim \text{Bernoulli}(q(w,x))} \sbr{h (w, x, y) Z } \\
    =&
    \EE_{(x,y) \sim D} \sbr{h (w, x, y) q(w,x)  } \\
    =&
    \EE_{(x,y) \sim D} \sbr{ - \frac{1}{\sigma} y \rbr{\frac{x}{\|w\|_2} - \frac{\inner{w}{x} w}{\|w\|_2^3} } \sigma \abs{\phi_\sigma'(\inner{\frac{w}{\|w\|}}{x} ) }   } \\
    =&
    \nabla L_\sigma(w)        
    \end{align*}
    where the first equality is by the definition of \afo, the second equality uses the tower rule,  the third equality plugs in the value of $h (w, x, y) $ and $ q(w,x)$, the last equality is by the definition of $L_\sigma(w)  $.

    Now we prove the third item. 
    \begin{enumerate}
        \item If $\sigma < \frac1e$. 
    
    Let $C$ below be from Lemma~\ref{lem:deriv-p-moment}. 
    We have for any $w$ such that $\|w\| \geq 1$, for any $p, q \geq 1$ such that $\frac1p + \frac1q = 1$, 
    \begin{align*}
        \EE \sbr{\|g_w \|^2}
        =&
        \EE_{(x,y) \sim D, Z\sim \text{Bernoulli}(q(w,x))} \sbr{ \|h (w, x, y) Z \|^2 } \\
        =&
        \EE_{(x,y)\sim D} \sbr{ q(w, x) \| h(w, x, y) \|_2^2 } \\
        =&
        \EE_{(x,y)\sim D}  \sbr{ \sigma \abs{\phi_\sigma'(\inner{\frac{w}{\|w\|}}{x} ) } 
        \frac{1}{\sigma^2} \|\frac{x}{\|w\|_2} - \frac{\inner{w}{x} w}{\|w\|_2^3}\|^2 } \\
        =&
        \frac{1}{\sigma}
        \EE_{(x,y)\sim D} \abs{\phi_\sigma'(\inner{\frac{w}{\|w\|}}{x})} \|\frac{x}{\|w\|_2} - \frac{\inner{w}{x} w}{\|w\|_2^3}\|^2 \\
        \leq&
        \frac{1}{\sigma}
        \EE_{(x,y)\sim D} \abs{\phi_\sigma'(\inner{\frac{w}{\|w\|}}{x})} \|x\|^2 \\
        \leq&
        \frac{1}{\sigma}
        \rbr{\EE_{(x,y)\sim D} \abs{\phi_\sigma'(\inner{\frac{w}{\|w\|}}{x})}^p}^\frac{1}{p} \rbr{\EE_{(x,y)\sim D} \|x\|^{2q}}^\frac{1}{q} \\
        \leq&
        \frac{1}{\sigma}
        \rbr{C \frac{1}{\sigma^{p-1}}\ln \frac{1}{\sigma}}^\frac{1}{p} 
        \rbr{\Gamma(2q+1) e \beta^{2q} d^q}^\frac{1}{q} \\
        =&
        \frac{1}{\sigma}
        \tilde{O} ((\frac{1}{\sigma})^{1-\frac{1}{p}} \cdot q^2 d) \\
        =&
        \frac{1}{\sigma}
        \tilde{O} ((\frac{1}{\sigma})^{\frac{1}{q}} \cdot q^2 d)
    \end{align*}
    where the first equality is by the definition of $g_w$ in
    Algorithm~\ref{alg:first-order-oracle},
    the second equality uses the tower rule, the third equality is by the definition of $q(w,x)$ and $h(w,x,y)$, the fourth equality is by algebra, 
    the fifth inequality is because for all $x \in \RR^d$, \[
    \|\frac{x}{\|w\|_2} - \frac{\inner{w}{x} w}{\|w\|_2^3}\|^2 
    \leq
    \|\frac{x}{\|w\|_2}\|^2 
    \leq
    \|x\|^2
    \]
    The sixth inequality is by Holder's inequality. The seventh inequality is by Lemmas~\ref{lem:deriv-p-moment} and~\ref{lem:x^q-ub}. The eighth equality is by Lemma~\ref{lem:gamma-function-ub}. The ninth equality uses that $\frac1p + \frac1q = 1$. 
    
    Choosing $q = \ln \frac{1}{\sigma}$, we have $q > 1$ since $\sigma < \frac1e$. 
    
    we have \[
    \EE_{(x,y) \sim D} \sbr{\|g_w \|^2}
    \leq
    \frac{1}{\sigma}
    \tilde{O} ((\frac{1}{\sigma})^{\frac{1}{\ln \frac{1}{\sigma}}} \cdot (\ln \frac{1}{\sigma})^2 d)
    =
    \frac{1}{\sigma}
    \tilde{O} (\exp(\ln \frac{1}{\sigma} \cdot \frac{1}{\ln \frac{1}{\sigma}}) \cdot (\ln \frac{1}{\sigma})^2 d)
    =
    \tilde{O} (\frac{d}{\sigma})
    \]
    where the last two equalities are by algebra. 

    \item If $\sigma \geq \frac1e$. Then we can proceed as follows, 
    \begin{align*}
        \EE \sbr{\|g_w \|^2}
        =&
        \EE_{(x,y) \sim D, Z\sim \text{Bernoulli}(q(w,x))} \sbr{ \|h (w, x, y) Z \|^2 } \\
        =&
        \EE_{(x,y)\sim D} \sbr{ q(w, x) \| h(w, x, y) \|_2^2 } \\
        =&
        \EE_{(x,y)\sim D}  \sbr{ \sigma \abs{\phi_\sigma'(\inner{\frac{w}{\|w\|}}{x} ) } 
        \frac{1}{\sigma^2} \|\frac{x}{\|w\|_2} - \frac{\inner{w}{x} w}{\|w\|_2^3}\|^2 } \\
        =&
        \frac{1}{\sigma}
        \EE_{(x,y)\sim D} \abs{\phi_\sigma'(\inner{\frac{w}{\|w\|}}{x})} \|\frac{x}{\|w\|_2} - \frac{\inner{w}{x} w}{\|w\|_2^3}\|^2 \\
        \leq&
        \frac{1}{\sigma}
        \EE_{(x,y)\sim D} \abs{\phi_\sigma'(\inner{\frac{w}{\|w\|}}{x})} \|x\|^2 \\
        \leq&
        \frac{1}{\sigma^2} \EE_{(x,y)\sim D} \|x\|^2 \\
        \leq&
        \frac{1}{\sigma^2} O(d) \\
        =& O(d)
    \end{align*}
    where the sixth inequality is because $\abs{\phi_\sigma'(t)} \leq \frac{1}{\sigma}$, for all $t \in \RR$, the seventh inequality uses Lemma~\ref{lem:x^q-ub} with $q = 2$, the eighth inequality uses $\sigma \geq \frac1e$. 
    
    \end{enumerate}

    Lastly, we prove the fourth item. We have, 
    \[
    \EE_{(x,y) \sim D, Z\sim \text{Bernoulli}(q(w,x))} \sbr{ Z }
    =
    \EE_{(x,y)\sim D} \sbr{ q(w, x)  }
    =
    \sigma \EE_{(x,y)\sim D} \abs{\phi_\sigma'(\inner{\frac{w}{\|w\|}}{x} ) }
    \leq
    \sigma C \ln \frac{1}{\sigma}
    =
    \tilde{O}(\sigma)
    \]
    where the first equality uses the tower rule, the second equality uses the definition of $q(w, x)$, the inequality is by applying Lemma~\ref{lem:deriv-p-moment} with $p = 1$. 
\end{proof}

\begin{lemma}
    With probability at least $1-\delta$, the total number of label queries after $T$ calls to \afo is at most $\tilde{O}(\sigma T + \sqrt{T \ln \frac{1}{\delta}} )$.
    \label{lem:high-prob-label}
\end{lemma}
\begin{proof}
    Let $Z_i$ be the query indicator the $i$-th time \afo is called.  
    Define a filtration $\cbr{\Gcal_i}_{i=0}^N$, where $\Gcal_i$ denotes the $\sigma$-field $\sigma(w_1, Z_1, w_2, Z_2, \ldots, w_i, Z_i)$. In this proof, We use $\EE_i [\cdot]$ to denote the conditional expectation with respect to $\Gcal_i$.

    Let $M_i = \sum_{j=1}^i Z_j - \EE\sbr{Z_j \mid \Gcal_{j-1}}$.
    It can be seen that $\cbr{M_i}_{i=1}^T$ is a martingale and $|M_i - M_{i-1}| \leq 1$. 
    Applying Azuma's inequality, we have that with probability at least $1-\delta$,
    \[
    M_T = \sum_{j=1}^T Z_j - \sum_{j=1}^T \EE\sbr{Z_j \mid \Gcal_{j-1}} \leq \sqrt{ 2 T \ln \frac{1}{\delta} }. 
    \]

    In addition, by item~\ref{item:expected-label} of Lemma~\ref{lem:active-oracle-property}, 
    $\EE\sbr{Z_j \mid \Gcal_{j-1}} = \tilde{O}(\sigma)$ for all $j \in \cbr{1,\ldots,T}$. Combining with the above inequality, we conclude that 
    \[
    \sum_{j=1}^T Z_j 
    \leq
    \sum_{j=1}^T \EE\sbr{Z_j \mid \Gcal_{j-1}} + \sqrt{ 2 T \ln \frac{1}{\delta} }
    \leq 
    \tilde{O}(\sigma T + \sqrt{T \ln\frac1\delta}).
    \]

\end{proof}

\section{Auxiliary lemmas}

\begin{lemma}
Let $D_X$ be a well-behaved distribution, 
then there exists a constant $C$, such that
for any $w$ such that $\|w\| \geq 1$, and any $p \geq 1$, 
\[
\EE_{x \sim D_X} \abs{\phi_\sigma'(\inner{\frac{w}{\|w\|}}{x} ) }^p 
\leq 
C \frac{1}{\sigma^{p-1}}\ln \frac{1}{\sigma}
\]
\label{lem:deriv-p-moment}
\end{lemma}

\begin{proof}
    For $k \in \cbr{0} \cup \NN$,
    denote by $R_k = \cbr{x \in \RR^d: \abs{\inner{\frac{w}{\|w\|}}{x} } \in (k\sigma, (k+1)\sigma)}$, then we have 
    \begin{align*}
    \EE_{x \sim D_X} \abs{\phi_\sigma'(\inner{\frac{w}{\|w\|}}{x})}^p
    =&
    \sum_{k=0}^\infty \EE_{x \sim D_X} \abs{\phi_\sigma'(\inner{\frac{w}{\|w\|}}{x})}^p \ind(x \in R_k) \\
    \leq&
    \sum_{k=0}^\infty \rbr{\frac{1}{\sigma} \frac{e^k}{(1+e^k)^2}}^p \PP(x \in R_k) \\
    \leq&
    \frac{1}{\sigma^p}
    4 \sigma U \beta \ln \frac{2}{\sigma U \beta}
    \sum_{k=0}^\infty \rbr{\frac{e^k}{(1+e^k)^2}}^p \\
    =&
    \frac{1}{\sigma^{p-1}} 4 U \beta \ln \frac{2}{\sigma U \beta}
    \sum_{k=0}^\infty \rbr{\frac{e^k}{(1+e^k)^2}}^p \\
    \leq& C
    \frac{1}{\sigma^{p-1}} \ln \frac{1}{\sigma}
    \end{align*}
    where the first equality is by the partition of $\RR^d$, namely $\RR^d = \cup_{k=0}^\infty R_k$. The second inequality is because for $k \in \cbr{0} \cup\NN$, if $|t| \in (k\sigma, (k+1)\sigma)$, $\abs{\phi'_\sigma(t)} = \abs{\frac{1}{\sigma}\frac{e^{\frac{t}{\sigma}}}{(1 +e^{\frac{t}{\sigma}})^2}} 
    \leq \frac{1}{\sigma} \frac{e^k}{(1+e^k)^2}$. The third inequality is by Lemma~\ref{lem:prob-of-slice}. 
    The constant $C$ in the fourth equality exists, because \[
    \frac{\frac{e^{k+1}}{(1+e^{k+1})^2}}{\frac{e^k}{(1+e^k)^2}}
    =
    e \frac{(1+e^k)^2}{(1+e^{k+1})^2}
    \leq
    e \frac{(1+e^0)^2}{(1+e^1)^2}
    <1, \forall k= 0,1,2,\ldots
    \]
    This means $\cbr{\frac{e^k}{(1+e^k)^2}: k= 0,1,2,\ldots}$ is decaying faster than a convergent power series, and thus this sequence is also summable. For all $p \geq 1$, $\rbr{\frac{e^k}{(1+e^k)^2}}^p \leq \frac{e^k}{(1+e^k)^2}$, so $\sum_{k=0}^\infty \rbr{\frac{e^k}{(1+e^k)^2}}^p \leq \sum_{k=0}^\infty \frac{e^k}{(1+e^k)^2}$. 
    
\end{proof}

\begin{lemma}
Let $D_X$ be a well-behaved distribution, then for all $q \geq 2$, we have $\EE_{x \sim D_X} \|x\|^q \leq \Gamma(q+1) e \beta^q d^\frac{q}{2}$.
\label{lem:x^q-ub}
\end{lemma}
\begin{proof}
    By Holder's inequality, with $p' = \frac{q}{2}, q' = \frac{q}{q-2}$, 
    \[
    \sum_{i=1}^d x_i^2 
    =
    \sum_{i=1}^d x_i^2 \cdot 1
    \leq
    \rbr{\sum_{i=1}^d |x_i|^q}^\frac{2}{q} \rbr{\sum_{i=1}^d 1}^\frac{q-2}{q}
    =
    \rbr{\sum_{i=1}^d |x_i|^q}^\frac{2}{q} d^\frac{q-2}{q}
    \]
    Hence \[
    \|x\|^q
    =
    \rbr{\sum_{i=1}^d x_i^2 }^{\frac{q}{2}}
    \leq
    \rbr{\sum_{i=1}^d |x_i|^q} d^\frac{q-2}{2}
    \]
    \[
    \EE_{x \sim D_X} \|x\|^q
    \leq
    \EE_{x \sim D_X} \rbr{\sum_{i=1}^d |x_i|^q} d^\frac{q-2}{2}
    =
    d^\frac{q-2}{2} \sum_{i=1}^d \EE |x_i|^q
    \]
    
    For all $i \in \sbr{d}, q \geq 2$, \[
    \EE |x_i|^q
    =
    \int_0^\infty \PP(|x_i|^q > t) \dif t
    =
    \int_0^\infty \PP(|x_i| > t^{\frac{1}{q}}) \dif t
    \leq 
    \int_0^\infty \exp(1-t^{\frac{1}{q}} / \beta) \dif t
    =
    \Gamma(q+1) e \beta^q
    \]
    where the first equality computes the expectation by integrating over the tail probability, the inequality is by the definition of well-behaved distribution, 
    the last equality can be calculated by the definition of the Gamma function: let $x = t^{\frac{1}{q}} / \beta$, then 
    $t = x^q \beta^q$ and $\dif t = x^{q-1} q \beta^{q} \dif x$, and
    \[
    \int_0^\infty \exp(1-t^{\frac{1}{q}} / \beta) \dif t
    =
    e \int_0^\infty e^{-x} x^{q-1} q \beta^{q} \dif x
    =
    e q \beta^{q} \Gamma(q)
    =
    e \beta^{q} \Gamma(q + 1)
    \]
    
    Hence $\EE_{x \sim D_X}  \|x\|^q \leq \Gamma(q+1) e \beta^q d^\frac{q}{2}$.     
\end{proof}

\begin{claim}
    Let $D_X$ be a well-behaved distribution, then for any unit vector $w$ and for all $q \geq 1$, $\EE_{x \sim D_X} \abs{\inner{x}{v}}^{q} \leq \Gamma(q+1) e \beta^q$.  
    \label{claim:marginalization}
\end{claim}
\begin{proof}
    \[
    \EE \abs{\inner{x}{v}}^{q}
    =
    \int_0^\infty \PP(\abs{\inner{x}{v}}^q > t) \dif t
    =
    \int_0^\infty \PP(\abs{\inner{x}{v}} > t^{\frac{1}{q}}) \dif t
    \leq 
    \int_0^\infty \exp(1-t^{\frac{1}{q}} / \beta) \dif t
    =
    \Gamma(q+1) e \beta^q
    \]
    where the first equality computes the expectation by integrating over the tail probability, the inequality is by the definition of well-behaved distribution, the other equalities and inequalities are by algebra. 
\end{proof}

\begin{lemma}
Let $D_X$ be a well-behaved distribution, then for any unit vector $w$, any $b_0 \geq 0, b>0$, we have \[
\PP_{x \sim D_X} (b_0 < |\inner{w}{x} | < b_0 +b) 
\leq 
4 b U \beta \ln \frac{2}{b U \beta}
\]
\label{lem:prob-of-slice}
\end{lemma}
\begin{proof}
    WLOG, assume $w = (1, 0, \ldots, 0)$, then $b_0 < |\inner{w}{x} | < b_0 +b$ is equivalent to $b_0 < |x_1| < b_0 +b$. For any $\gamma>0$, by the definition of well-behaved distribution, \[
    \PP(b_0 < |x_1| < b_0 +b)
    \leq
    \PP(b_0 < |x_1| < b_0 +b , |x_2| \leq \beta \ln \frac{e}{\gamma}) + \PP(|x_2| \geq \beta \ln \frac{e}{\gamma})
    \leq
    4bU\beta \ln \frac{e}{\gamma} +\gamma 
    \]
    where the first inequality is because $\cbr{x: b_0 < |x_1| < b_0 +b} \subset \cbr{x: b_0 < |x_1| < b_0 +b , |x_2| \leq \beta \ln \frac{e}{\gamma} } \cup \cbr{x: |x_2| \geq \beta \ln \frac{e}{\gamma}}$, 
    the second inequality is by the definition of well-behaved distribution. 
    
    Taking $\gamma = 4bU\beta$, we have \[
    \PP_{x \sim D_X} (b_0 < |\inner{w}{x} | < b_0 +b) 
    \leq 
    4bU\beta \rbr{\ln \frac{e}{4bU\beta} +1} 
    \leq
    4 b U \beta \ln \frac{2}{b U \beta}
    \]
\end{proof}

\begin{lemma}
$|\phi_\sigma''(t)| \leq \frac{1}{\sigma} |\phi_\sigma'(t)|$ for all $t \in \RR$. 
\label{lem:second-deriv}
\end{lemma}
\begin{proof}
    Since $\phi_\sigma(t) = \frac{1}{1 +e^{\frac{t}{\sigma}}}$, by direct calculation, $\phi_\sigma'(t) = -\frac{1}{\sigma}\frac{e^{\frac{t}{\sigma}}}{(1 +e^{\frac{t}{\sigma}})^2}$, and 
    $\phi_\sigma''(t) = \frac{1}{\sigma^2}\frac{e^{\frac{t}{\sigma}} (e^{\frac{t}{\sigma}} - 1) }{(1 +e^{\frac{t}{\sigma}})^3}$. Hence for all $t \in \RR$, we have \[
    |\phi_\sigma''(t)| 
    =
    \frac{1}{\sigma^2}\frac{e^{\frac{t}{\sigma}} |e^{\frac{t}{\sigma}} - 1| }{(1 +e^{\frac{t}{\sigma}})^3}
    \leq
    \frac{1}{\sigma^2}\frac{e^{\frac{t}{\sigma}} }{(1 +e^{\frac{t}{\sigma}})^2}
    =
    \frac{1}{\sigma} |\phi_\sigma'(t)|
    \]
    where the inequality is by the elementary fact that $|a-1| \leq a+1$ for all $a \geq 0$. 
\end{proof}

\begin{lemma}
For all $w$ such that $\|w\| \geq 1$, $\|\nabla_w^2 L_\sigma(w) \| = \tilde{O}(\frac1\sigma)$. 
\label{lem:smoothness-factor-ub}
\end{lemma}
\begin{proof}
    Throughout this proof, we denote by $\ell(w,x) := \inner{\frac{w}{\|w\|}}{x}$. 
    We continue the calculation of $ \nabla^2 L_\sigma(w)$ in Lemma B.2 of~\cite{diakonikolas2020learning} and refine the result therein.
    \begin{align}
        \nabla^2 \phi_\sigma(y \ell(w,x)) 
        =&
        \phi_\sigma''(y \ell(w,x)) \rbr{\frac{xx^\top}{\|w\|^2} - \frac{\inner{w}{x}}{\|w\|^4}wx^\top - \frac{\inner{w}{x}}{\|w\|^4}xw^\top 
        + \frac{\inner{w}{x}^2}{\|w\|^6}ww^\top}  \nonumber\\
        & + \phi_\sigma'(y \ell(w,x)) y \nabla^2 \ell(w,x) 
        \label{eqn:hessian}
    \end{align}
    
    Our goal here is to upper bound $\|\EE_{(x,y) \sim D}\nabla^2 \phi_\sigma(y \ell(w,x))\|_{\mathrm{op}}$. By triangle inequality, it suffices to upper bound the operator norm of each individual term. 
    
    Let $v \in \mathbb{S}^{d-1}$, $C$ below be from Lemma~\ref{lem:deriv-p-moment}. 

    For any $w$ such that $\|w\| \geq 1$, for any $p, q \geq 1$ such that $\frac1p + \frac1q = 1$, 
    \begin{align*}
        \abs{ \inner{v}{\EE_{(x,y) \sim D} \sbr{ \phi_\sigma''(y \ell(w,x)) \frac{xx^\top}{\|w\|^2}} v } }
        \leq&
        \EE_{(x,y) \sim D} \sbr{\frac{\abs{\phi_\sigma''(y \ell(w,x))}}{\|w\|^2} \inner{x}{v}^2 } \\
        \leq&
        \frac{1}{\sigma} \EE_{(x,y) \sim D} \sbr{\abs{\phi_\sigma'(\ell(w,x))}  \inner{x}{v}^2 } \\
        \leq&
        \frac{1}{\sigma}
        \sbr{\EE_{(x,y)\sim D} \abs{\phi_\sigma'(\inner{\frac{w}{\|w\|}}{x})}^p }^\frac{1}{p}
        \sbr{\EE_{(x,y)\sim D} \abs{\inner{x}{v}}^{2q} }^\frac{1}{q} \\
        \leq&
        \frac{1}{\sigma} \rbr{C \frac{1}{\sigma^{p-1}} \ln \frac{1}{\sigma} }^\frac{1}{p} \rbr{\Gamma(2q+1) e \beta^{2q}}^\frac{1}{q} \\
        =&
        \frac{1}{\sigma} \tilde{O} ((\frac{1}{\sigma})^{1-\frac{1}{p}}  \cdot q^2 ) \\
        =&
        \frac{1}{\sigma} \tilde{O} ((\frac{1}{\sigma})^{\frac{1}{q}}  \cdot q^2 )
    \end{align*}
    where the first inequality is moving the absolute value inside the expectation, the second inequality uses Lemma~\ref{lem:second-deriv} and $\phi_\sigma'$ is even, as well as $\|w\| \geq 1$, the third inequality is by Holder's inequality. The fourth inequality is by Lemmas~\ref{lem:deriv-p-moment} and Claim~\ref{claim:marginalization}. The fifth equality is by Lemma~\ref{lem:gamma-function-ub}. The sixth equality uses that $\frac1p + \frac1q = 1$. 
    
    Taking $q = \ln \frac{1}{\sigma}$, we have \[
    \abs{ \inner{v}{\EE_{(x,y) \sim D} \sbr{ \phi_\sigma''(y \ell(w,x)) \frac{xx^\top}{\|w\|^2}} v } }
    \leq
    \frac{1}{\sigma} \tilde{O} ((\frac{1}{\sigma})^{\frac{1}{q}}  \cdot q^2 )
    =
    \tilde{O}(\frac1\sigma)
    \]
    
    Similarly, 
    \begin{align*}
        &\abs{ \inner{v}{\EE_{(x,y) \sim D} \sbr{ \phi_\sigma''(y \ell(w,x)) \frac{\inner{w}{x}}{\|w\|^4}wx^\top } v } } \\
        \leq&
        \EE_{(x,y) \sim D} \sbr{\frac{\abs{\phi_\sigma''(y \ell(w,x))}}{\|w\|^4} |\inner{w}{x}| |\inner{v}{w}|  |\inner{x}{v}| } \\
        \leq&
        \EE_{(x,y) \sim D} \sbr{\frac{\abs{\phi_\sigma''(y \ell(w,x))}}{\|w\|^3} |\inner{w}{x}|  |\inner{x}{v}| } \\
        =&
        \EE_{(x,y) \sim D} \sbr{\frac{\abs{\phi_\sigma''(y \ell(w,x))}}{\|w\|^2} |\inner{ \frac{w}{\|w\|} }{x}|  |\inner{x}{v}| } \\
        \leq&
        \frac{1}{\sigma} \EE_{(x,y) \sim D} \sbr{\abs{\phi_\sigma'(\ell(w,x))}  |\inner{ \frac{w}{\|w\|} }{x}|  |\inner{x}{v}| } \\
        \leq&
        \frac{1}{\sigma}
        \sbr{\EE_{(x,y)\sim D} \abs{\phi_\sigma'(\inner{\frac{w}{\|w\|}}{x})}^p }^\frac{1}{p}
        \sbr{\EE_{(x,y)\sim D} \abs{\inner{ \frac{w}{\|w\|} }{x}}^{2q} }^\frac{1}{2q}
        \sbr{\EE_{(x,y)\sim D} \abs{\inner{x}{v}}^{2q} }^\frac{1}{2q} \\
        \leq&
        \frac{1}{\sigma} \rbr{\frac{1}{\sigma^{p-1}} \ln \frac{1}{\sigma} }^\frac{1}{p} \rbr{\Gamma(2q+1) e \beta^{2q}}^\frac{1}{2q} \rbr{\Gamma(2q+1) e \beta^{2q}}^\frac{1}{2q} \\
        =&
        \frac{1}{\sigma} \tilde{O} ((\frac{1}{\sigma})^{1-\frac{1}{p}}  \cdot q^2 )
    \end{align*}
    
    Again, taking $q = \ln \frac{1}{\sigma}$, we have \[
    \abs{ \inner{v}{\EE_{(x,y) \sim D} \sbr{ \phi_\sigma''(y \ell(w,x)) \frac{\inner{w}{x}}{\|w\|^4}wx^\top } v } }
    \leq
    \frac{1}{\sigma} \tilde{O} ((\frac{1}{\sigma})^{\frac{1}{q}}  \cdot q^2 )
    =
    \tilde{O}(\frac1\sigma)
    \]
    
    The same calculation and the upper bound goes for $\abs{ \inner{v}{\EE_{(x,y) \sim D} \sbr{ \phi_\sigma''(y \ell(w,x)) \frac{\inner{w}{x}}{\|w\|^4}xw^\top  } v } }$.
    
    For the fourth term in Eqn~\eqref{eqn:hessian}, 
    \begin{align*}
        \abs{ \inner{v}{\EE_{(x,y) \sim D} \sbr{ \phi_\sigma''(y \ell(w,x)) \frac{\inner{w}{x}^2}{\|w\|^6}ww^\top } v } }
        \leq&
        \EE_{(x,y) \sim D} \sbr{\frac{\abs{\phi_\sigma''(y \ell(w,x))}}{\|w\|^6} \inner{w}{x}^2 \inner{v}{w}^2 } \\
        =&
        \EE_{(x,y) \sim D} \sbr{\frac{\abs{\phi_\sigma''(y \ell(w,x))}}{\|w\|^4} \inner{ \frac{w}{\|w\|}  }{x}^2 \inner{v}{w}^2 } \\
        \leq&
        \EE_{(x,y) \sim D} \sbr{\frac{\abs{\phi_\sigma''(y \ell(w,x))}}{\|w\|^2} \inner{ \frac{w}{\|w\|}  }{x}^2  } \\
        \leq&
        \frac{1}{\sigma} \EE_{(x,y) \sim D} \sbr{\abs{\phi_\sigma'(\ell(w,x))}  \inner{ \frac{w}{\|w\|}  }{x}^2 } 
    \end{align*}
    where the first inequality is moving the absolute value inside the expectation, the second inequality 
    is by algebra, 
    the third inequality 
    uses $\abs{\inner{ \frac{w}{\|w\|}  }{v} }\leq 1$, the fourth inequality uses Lemma~\ref{lem:second-deriv} and $\phi_\sigma'$ is even, as well as $\|w\| \geq 1$. 
    
    It follows the same upper bound \[
    \abs{ \inner{v}{\EE_{(x,y) \sim D} \sbr{ \phi_\sigma''(y \ell(w,x)) \frac{\inner{w}{x}^2}{\|w\|^6}ww^\top } v } }
    \leq
    \tilde{O}(\frac1\sigma)
    \]

    To upper bound the operator norm of the last term in Eqn~\eqref{eqn:hessian}, note that $\phi_\sigma'(t) \leq \frac{1}{\sigma}$, for all $t \in \RR$, and \[
    \nabla^2 \ell(w,x)
    =
    - \frac{xw^T}{\|w\|^3} - \frac{wx^T}{\|w\|^3} - \frac{\inner{w}{x}}{\|w\|^3} I
    + 3\frac{\inner{w}{x}}{\|w\|^5}ww^T
    \]
    Then we can upper bound the operator norm for each individual term, 
    \begin{align*}
    \abs{ \inner{v}{\EE_{(x,y) \sim D} \sbr{  \frac{xw^\top}{\|w\|^3}} v } }
    \leq&
    \EE_{(x,y) \sim D} \sbr{\frac{1}{\|w\|^3} \abs{\inner{x}{v} \inner{w}{v} }} \\
    \leq&
    \EE_{(x,y) \sim D} \sbr{\frac{1}{\|w\|^2} \abs{\inner{x}{v} } } \\
    \leq&
    \EE_{(x,y) \sim D} \sbr{ \abs{\inner{x}{v} } } \\
    =& 
    O(1)
    \end{align*}
    where the first inequality is moving the absolute value inside the expectation, the second inequality uses $\abs{\inner{ \frac{w}{\|w\|}  }{v} }\leq 1$, the third inequality uses $\|w\| \geq 1$, the last equality applies Claim~\ref{claim:marginalization} with $q = 1$. 
    
    Similarly, 
    \begin{align*}
    \abs{ \inner{v}{\EE_{(x,y) \sim D} \sbr{  \frac{\inner{w}{x}}{\|w\|^3} I }  v } }
    \leq&
    \EE_{(x,y) \sim D} \sbr{\frac{1}{\|w\|^2} \abs{\inner{ \frac{w}{\|w\|}  }{x} }} \\
    \leq&
    \EE_{(x,y) \sim D} \sbr{ \abs{\inner{ \frac{w}{\|w\|}  }{x} } } \\
    =& 
    O(1)
    \end{align*}
    
    Putting above terms together, we have \[
    \|\nabla_w^2 L_\sigma(w) \|_{\mathrm{op}} = \tilde{O}(\frac1\sigma)
    \]
\end{proof}

\begin{lemma}[Stochastic gradient is sub-exponential]
Let $g_w$ be the random output of $\afo(w)$. Then for any unit vector $w$, $\|g_w\|$ is sub-exponential with parameter $K = \tilde{\Theta}(\frac{\sqrt{d}}{\sigma})$, that is, 
        the tails of $\|g_w\|$ satisfy \[
        \PP ( \|g_w\| \geq t)
        \leq
        2 \exp (-\frac{t}{K}), \forall t \geq 0
        \]
    \label{lem:sub-exponential}
\end{lemma}

\begin{proof}
    Assume WLOG that $w = (1, 0, \ldots, 0)$, 
    \begin{align*}
        \PP ( \|g_w\| \geq t)
        \leq&
        \PP ( \|h(w, x, y) \| \geq t) \\
        \leq&
        \PP ( \frac{1}{\sigma} \|x \| \geq t) \\
        =&
        \PP ( \|x \| \geq \sigma t) \\
        \leq&
        \PP ( \sqrt{d} \|x \|_\infty \geq \sigma t) \\
        \leq&
        d \cdot \PP ( |x_i| \geq \frac{\sigma t}{\sqrt{d}}) \\
        \leq&
        d \exp(-\frac{\sigma t}{\sqrt{d}})
    \end{align*}
    where the first inequality uses the fact that the events $\cbr{\|g_w\| \geq t} \subseteq \cbr{ \|h(w, x, y) \| \geq t}, \forall t \geq 0$, the second inequality uses that $\|h(w,x,y)\| \leq \frac{1}{\sigma} \|x \|, \forall x \in \RR^d$, so the events $\cbr{ \|h(w, x, y) \| \geq t} \subseteq \cbr{ \frac{1}{\sigma} \|x \| \geq t}, \forall t \geq 0$. The third equality is by algebra. The fourth inequality uses $\|x \| \leq \sqrt{d} \|x \|_\infty, \forall x \in \RR^d$. The fifth inequality uses a union bound on $d$ coordinates. The sixth inequality is by the definition of well-behaved distribution. 
    
    Therefore, by Claim~\ref{claim:sub-exp}, $\|g_w\|$ is $\Theta(\frac{\sqrt{d} \ln d}{\sigma})$ sub exponential. 
\end{proof}

\begin{claim}
    If $\PP (|X| \geq t) \leq 2C \exp(-\frac{t}{K})$ for some constant $C \geq e^2$, then $\PP (|X| \geq t) \leq 2\exp(-\frac{t}{2 K \ln C})$. 
    \label{claim:sub-exp}
\end{claim}
\begin{proof}
    Let $t_0 = \frac{2K \ln^2 C}{2 \ln C - 1}$.
    \begin{enumerate}
    \item If $t \leq t_0$, $\PP (|X| \geq t) \leq 1 \leq 2 \exp(-\frac{t_0}{2K \ln C}) \leq 2\exp(-\frac{t}{2 K \ln C})$. 

    The second inequality is true, because $t_0 = \frac{2K \ln^2 C}{2 \ln C - 1} \leq 2K \ln C \ln 2$, using $C \geq e^2$.

    \item If $t > t_0$, then $\PP (|X| \geq t) \leq 2C \exp(-\frac{t}{K}) = 2 \exp(\ln C -\frac{t}{K}) \leq 2 \exp(-\frac{t}{2K \ln C})$. 
    \end{enumerate}   

\end{proof}

\begin{lemma}
Suppose $D$ satisfies the $(A,\alpha)$-Tsybakov noise condition. Then
    the Bayes error $\err(w^*)$ satisfies $\err(w^*) \leq \frac12 -\alpha (\frac{1}{A})^\frac{1-\alpha}{\alpha}$.
    \label{lem:bayes-error}
\end{lemma}
\begin{proof}
    By the definition of Tsybakov noise, Definition~\ref{def:tnc}, we know $\PP (\eta(x) \geq \frac{1}{2} - t) \leq A t^\frac{\alpha}{1-\alpha}$, for $t \in [0,\frac12]$. 

    Taking $t = \frac12$, we can see $A, \alpha$ need to satisfy $1 = \PP (\eta(x) \geq 0) \leq A (\frac12)^\frac{\alpha}{1-\alpha}$. We proceed to calculate $\err(w^*)$ as follows. 
    \[
    \err(w^*) = \EE \eta(x)
    = \int_0^\infty \PP (\eta(x) \geq t) \dif t 
    = \int_0^\frac12 \PP (\eta(x) \geq t) \dif t 
    \leq \int_0^\frac12
    \min(1, A (\frac{1}{2} - t)^\frac{\alpha}{1-\alpha}) \dif t 
    \]
    where the first equality is by the definition of the Bayes classifier, the second equality is writing the expectation as the integral of the tail probability, the third equality uses $0 \leq \eta(x) \leq \frac12, \forall x \in \RR^d$, the inequality uses the trivial upper bound $1$ of the probability. 
    
    Let $t_0 = \frac12 - (\frac{1}{A})^\frac{1-\alpha}{\alpha}$, 
    so $1 = A (\frac{1}{2} - t_0)^\frac{\alpha}{1-\alpha}$, and we have, 
    \begin{align*}
        \int_0^\frac12 \min ( 1, A (\frac{1}{2} - t)^\frac{\alpha}{1-\alpha} ) \dif t 
        =&
        t_0 + \int_{t_0}^\frac12A (\frac{1}{2} - t)^\frac{\alpha}{1-\alpha} \dif t \\
        =& 
        \frac12 - (\frac{1}{A})^\frac{1-\alpha}{\alpha} + (1-\alpha)(\frac{1}{A})^\frac{1-\alpha}{\alpha} \\
        =&
        \frac12  + ((1-\alpha) - 1) (\frac{1}{A})^\frac{1-\alpha}{\alpha} \\
        =&
        \frac12 -\alpha (\frac{1}{A})^\frac{1-\alpha}{\alpha}
     \end{align*}
     Therefore, $\err(w^*) \leq \frac12 -\alpha (\frac{1}{A})^\frac{1-\alpha}{\alpha}$.
\end{proof}

\begin{lemma}[Theorem 2.1 in~\cite{juditsky2008large}]
Suppose martingale difference $\cbr{\xi_i}_{i=1}^\infty$ satisfies \[
\forall i \geq 1, \EE_{i-1} \cbr{\exp \cbr{\frac{\|\xi_i\|}{\nu} } } \leq \exp(1) \text{  almost surely}
\]
then, for all $N \geq 1$ and $\gamma \geq 0$, one has \[
\PP \cbr{ \|\sum_{i=1}^N \xi_i \| \geq \sqrt{2}( \sqrt{e}  + \gamma ) \sqrt{N} \nu     }
\leq 
2 \exp \cbr{ -\frac{1}{64} \min \sbr{\gamma^2, 16 \sqrt{N} \gamma } }
\]
\label{lem:sub-exp-concentration}
\end{lemma}

\begin{lemma}[Lemma 26 in~\cite{zhang2021improved}]
If $D$ is $(2, L, R, U, \beta)$-well behaved, then, we have for any $u, v$ in $\RR^d$, 
for all $\gamma > 0$, $\PP_{x \sim D_X}(h_u(x) \neq h_v(x)) \leq 4 U \beta^2 \rbr{\ln\frac{6}{\gamma}}^2 \theta(u,v) + \gamma$.
\label{lem:prob-angle}
\end{lemma}

\begin{lemma}
    There exists $c > 0$, s.t. for any $q \geq 1$, $\Gamma(2q+1)^{\frac1q} \leq c q^2$.
    \label{lem:gamma-function-ub}
\end{lemma}
\begin{proof}
We show that for any $m >0$, $\Gamma(m+1) \leq 3 \rbr{\frac{3m}{5}}^m$. 
The proof is originally from \url{https://math.stackexchange.com/questions/214422/bounding-the-gamma-function}; for completeness, we reproduce the proof here. 

Let $0 < \alpha <1, f(t) = e^{-\alpha t} t^m$ where $t >0 $, it can be checked (by taking a derivative) that $f(t)$ achieves the maximum at $t = \frac{m}{\alpha}$. Hence for any $m>0$, \[
\Gamma(m+1) = \int_0^\infty e^{-t} t^m \dif t
= \int_0^\infty e^{-\alpha t} t^m e^{-(1-\alpha)t} \dif t
\leq (\frac{m}{ae})^m \int_0^\infty  e^{-(1-\alpha)t} \dif t
= (\frac{m}{ae})^m \rbr{\frac{1}{1-\alpha}}
\]
where the first equality is by the definition of the Gamma function, and the other equalities and inequalities are by algebra.

Taking $\alpha = \frac{5}{3e}$ and noting that $\frac{1}{1-\frac{5}{3e}} \leq 3$, we obtain that for any $m >0$, $\Gamma(m+1) \leq 3 \rbr{\frac{3m}{5}}^m$. 

Therefore, for any $q \geq 1$, we have \[
\Gamma(2q+1)^{\frac1q} \leq 3^{\frac{1}{q}} \rbr{\frac{6q}{5}}^2 
\leq 3 \rbr{\frac{6q}{5}}^2 
\]
\end{proof}

We show in the following claim that 
\afo preserves the bound on the expected squared norm of the stochastic gradient
as passively querying the labels for all $x$. 
\begin{claim}
    For any unit vector $w$, 
    $\EE_{(x,y) \sim D} \sbr{\normx{\nabla \phi_\sigma \rbr{y \frac{\inner{w}{x}}{\|w\|}} }^2} \leq \tilde{O} (\frac{d}{\sigma})$.
    \label{claim:passive-variance}
\end{claim}
\begin{proof}
    We follow a similar proof idea of item~\ref{item:variance} in Lemma~\ref{lem:active-oracle-property-restate}.  
    \begin{enumerate}
        \item If $\sigma < \frac1e$.

    Let $C$ below be from Lemma~\ref{lem:deriv-p-moment}. 
    We have for any $w$ such that $\|w\| \geq 1$, for any $p, q \geq 1$ such that $\frac1p + \frac1q = 1$, 
    \begin{align*}
        \EE_{(x,y) \sim D} \sbr{\normx{\nabla \phi_\sigma \rbr{y \frac{\inner{w}{x}}{\|w\|}} }^2}
        =&
        \EE_{(x,y)\sim D} \abs{\phi_\sigma'(\inner{\frac{w}{\|w\|}}{x})}^2 \|\frac{x}{\|w\|_2} - \frac{\inner{w}{x} w}{\|w\|_2^3}\|^2 \\
        \leq&
        \EE_{(x,y)\sim D} \abs{\phi_\sigma'(\inner{\frac{w}{\|w\|}}{x})}^2 \|x\|^2 \\
        \leq&
        \rbr{\EE_{(x,y)\sim D} \abs{\phi_\sigma'(\inner{\frac{w}{\|w\|}}{x})}^{2p}}^\frac{1}{p} \rbr{\EE_{(x,y)\sim D} \|x\|^{2q}}^\frac{1}{q} \\
        \leq&
        \rbr{C \frac{1}{\sigma^{2p-1}}\ln \frac{1}{\sigma}}^\frac{1}{p} 
        \rbr{\Gamma(2q+1) e \beta^{2q} d^q}^\frac{1}{q} \\
        =&
        \tilde{O} ((\frac{1}{\sigma})^{2-\frac{1}{p}} \cdot q^2 d) \\
        =&
        \tilde{O} ((\frac{1}{\sigma})^{1 + \frac{1}{q}} \cdot q^2 d)
    \end{align*}
    where
    the first equality and second inequality are by algebra, 
    the third inequality is by Holder's inequality. The fourth inequality is by Lemmas~\ref{lem:deriv-p-moment} and~\ref{lem:x^q-ub}. The fifth equality is by Lemma~\ref{lem:gamma-function-ub}. The sixth equality uses that $\frac1p + \frac1q = 1$. 
    
     Choosing $q = \ln \frac{1}{\sigma}$, we have $q \geq 1$ since $\sigma \leq \frac1e$. 
    we have \[
    \EE_{(x,y) \sim D} \sbr{\normx{\nabla \phi_\sigma \rbr{y \frac{\inner{w}{x}}{\|w\|}} }^2}
    \leq
    \frac{1}{\sigma}
    \tilde{O} ((\frac{1}{\sigma})^{\frac{1}{\ln \frac{1}{\sigma}}} \cdot (\ln \frac{1}{\sigma})^2 d)
    =
    \frac{1}{\sigma}
    \tilde{O} (\exp(\ln \frac{1}{\sigma} \cdot \frac{1}{\ln \frac{1}{\sigma}}) \cdot (\ln \frac{1}{\sigma})^2 d)
    =
    \tilde{O} (\frac{d}{\sigma})
    \]
    where the last two equalities are by algebra. 
    \item 
    If $\sigma \geq \frac1e$. We proceed as follows. 
        \begin{align*}
        \EE_{(x,y) \sim D} \sbr{\normx{\nabla \phi_\sigma \rbr{y \frac{\inner{w}{x}}{\|w\|}} }^2}
        =&
        \EE_{(x,y)\sim D} \abs{\phi_\sigma'(\inner{\frac{w}{\|w\|}}{x})}^2 \|\frac{x}{\|w\|_2} - \frac{\inner{w}{x} w}{\|w\|_2^3}\|^2 \\
        \leq&
        \EE_{(x,y)\sim D} \abs{\phi_\sigma'(\inner{\frac{w}{\|w\|}}{x})}^2 \|x\|^2 \\
        \leq&
        \frac{1}{\sigma^2} \EE_{(x,y)\sim D} \|x\|^2 \\
        \leq&
        \frac{1}{\sigma^2} O(d) \\
        =& O(d)
    \end{align*}
    where the third inequality is because $\abs{\phi_\sigma'(t)} \leq \frac{1}{\sigma}$, for all $t \in \RR$, the fourth inequality uses Lemma~\ref{lem:x^q-ub} with $q = 2$, the fifth inequality uses $\sigma \geq \frac1e$. 
    \end{enumerate}
\end{proof}

\end{document}